\documentclass{article}

\usepackage{microtype}
\usepackage{graphicx}
\usepackage{subcaption}
\usepackage{booktabs,color,colortbl} 

\usepackage{pifont}
\newcommand{\xmark}{\ding{55}}

\usepackage{hyperref}

\usepackage[accepted]{icml2023}

\usepackage{amsmath}
\usepackage{amssymb}
\usepackage{mathtools}
\usepackage{amsthm}

\usepackage[capitalize,noabbrev]{cleveref}

\theoremstyle{plain}
\newtheorem{theorem}{Theorem}[section]

\theoremstyle{definition}
\newtheorem{definition}[theorem]{Definition}

\theoremstyle{remark}

\usepackage[textsize=tiny]{todonotes}

\DeclareRobustCommand{\quotes}[1]{``#1''}

\icmltitlerunning{No-Regret Reinforcement Learning in Smooth MDPs}

\allowdisplaybreaks[4]

\RequirePackage{mathtools}
\RequirePackage{amsthm}  
\RequirePackage{amsmath}
\RequirePackage{amssymb}
\RequirePackage{bm}
\RequirePackage{mathabx}
\RequirePackage{xspace}
\RequirePackage{twoopt}
\RequirePackage{mathrsfs} 
\RequirePackage{xifthen}
\RequirePackage{etoolbox}
\RequirePackage{dsfont}
\RequirePackage{ifthen}

\usepackage{enumitem}
\usepackage{thmtools}

\usepackage{thmtools}
\usepackage{thm-restate}

\newtheorem{thm}{Theorem}

\newtheorem{lem}{Lemma}

\newtheorem{prop}[thm]{Proposition}

\newtheorem{ass}{Assumption}

\newcommand{\R}{\mathbb{R}}
\newcommand{\N}{\mathbb{N}}
\newcommand{\E}{\mathop{\mathbb{E}}}

\newcommand{\wass}{\mathcal W}
\newcommand{\bigo}{\mathcal O}
\newcommand{\bigot}{\widetilde{\mathcal O}}

\makeatletter
\newcommand*{\bdiv}{%
  \nonscript\mskip-\medmuskip\mkern5mu%
  \mathbin{\operator@font div}\penalty900\mkern5mu%
  \nonscript\mskip-\medmuskip
}
\makeatother

\newcommand{\Ss}{\mathcal{S}}
\newcommand{\As}{\mathcal{A}}
\newcommand{\Xs}{\mathcal{X}}

\newcommand{\tv}{\mathrm{TV}}

\usepackage{xcolor}
\definecolor{brightBlue}{RGB}{68, 119, 170}
\definecolor{brightCyan}{RGB}{102, 204, 238}
\definecolor{brightGreen}{RGB}{34, 136, 51}
\definecolor{brightYellow}{RGB}{204, 187, 68}
\definecolor{brightRed}{RGB}{238, 102, 119}
\definecolor{brightPurple}{RGB}{170, 51, 119}
\definecolor{brightGrey}{RGB}{187, 187, 187}

\definecolor{vibrantBlue}{RGB}{0, 119, 187}
\definecolor{vibrantCyan}{RGB}{51, 187, 238}
\definecolor{vibrantTeal}{RGB}{0, 153, 136}
\definecolor{vibrantOrange}{RGB}{238, 119, 51}
\definecolor{vibrantRed}{RGB}{204, 51, 17}
\definecolor{vibrantMagenta}{RGB}{238, 51, 119}
\definecolor{vibrantGrey}{RGB}{100, 100, 100}

\definecolor {processblue}{cmyk}{0.96,0,0,0}
\definecolor {mygreen}{HTML}{008500}
\definecolor {myred}{HTML}{DD0000}
\definecolor {myyell}{HTML}{AAAA00}

\newcommand{\green}[1]{{\color{mygreen} #1}}

\begin{document}

\twocolumn[
\icmltitle{No-Regret Reinforcement Learning in Smooth MDPs}




\icmlsetsymbol{equal}{*}

\begin{icmlauthorlist}
\icmlauthor{Davide Maran$^1$}{}
\icmlauthor{Alberto Maria Metelli$^1$}{}
\icmlauthor{Matteo Papini$^1$}{}
\icmlauthor{Marcello Restelli$^1$}{}
\icmlauthor{$^1$Politecnico di Milano, Piazza Leonardo da Vinci, 32-36 - Città Studi, Milano (MI)}{}
\end{icmlauthorlist}


\icmlcorrespondingauthor{Firstname1 Lastname1}{first1.last1@xxx.edu}
\icmlcorrespondingauthor{Firstname2 Lastname2}{first2.last2@www.uk}

\icmlkeywords{Machine Learning, ICML}

\vskip 0.3in
]




\begin{abstract}
Obtaining no-regret guarantees for reinforcement learning (RL) in the case of problems with continuous state and/or action spaces is still one of the major open challenges in the field. Recently, a variety of solutions have been proposed, but besides very specific settings, the general problem remains unsolved. In this paper, we introduce a novel structural assumption on the Markov decision processes (MDPs), namely $\nu-$smoothness, that generalizes most of the settings proposed so far (e.g., linear MDPs and Lipschitz MDPs). 
To face this challenging scenario, we propose two algorithms for regret minimization in $\nu-$smooth MDPs. Both algorithms build upon the idea of constructing an MDP representation through an orthogonal feature map based on Legendre polynomials. The first algorithm, \textsc{Legendre-Eleanor}, archives the no-regret property under weaker assumptions but is computationally inefficient, whereas the second one, \textsc{Legendre-LSVI}, runs in polynomial time, although for a smaller class of problems. After analyzing their regret properties, we compare our results with state-of-the-art ones from RL theory, showing that our algorithms achieve the best guarantees.
\end{abstract}

\section{Introduction}
\emph{Reinforcement learning} (RL) \cite{sutton2018reinforcement} is a paradigm of artificial intelligence in which the agent interacts with an environment to maximize a reward signal in the long term. From the theoretical perspective, a lot of effort has been put into designing algorithms with small \emph{(cumulative) regret}, which is an index of how much the policies (i.e., the behavior) played by the algorithm during the learning process are suboptimal. For the case of \emph{tabular Markov decision processes} (MDPs), an optimal result was first proved by~\citet{azar2017minimax}, who showed a bound on the regret of order $\mathcal {\widetilde O}(H\sqrt{|\Ss||\As|K})$, where $\Ss$ is a finite state space, $\As$ is a finite action space, and $K$ is the number of episodes, and $H$ the time horizon of every episode. This regret is minimax-optimal, in the sense that no algorithm can achieve smaller regret for every arbitrary tabular MDP. Unfortunately, assuming that the state-action space is finite is extremely restrictive, as the number of states and/or actions can be huge or even infinite in practice. This is especially critical for a large variety of real-world scenarios in which RL has achieved successful results, including robotics \cite{kober2013reinforcement}, autonomous driving \cite{kiran2021deep}, and trading \cite{hambly2023recent}. These scenarios are usually modeled as MDPs with continuous state and/or action spaces, as the underlying dynamics is too complex to be captured by a finite number of states and/or actions. It is not by chance that one of the most common benchmarks for RL algorithms, 
\textsc{MuJoCo} \citep{todorov2012mujoco,brockman2016openai}, is composed of environments characterized by continuous state and action spaces. This highlights the notable gap between the current maturity of theory and the pressing needs of practical application. 
For this reason, devising algorithms with regret bounds for RL in \emph{continuous} spaces is currently one of the most important challenges of the whole field.

Since, without any further assumption, the RL problem in continuous spaces is \emph{non-learnable},\footnote{Think for example at searching for a maximum of a noisy reward function with infinitely many jumps.} the modern literature revolves around searching for the weakest \emph{structural assumptions} under which the problem can be solved efficiently. 
\emph{Linear quadratic regulator} (LQR) \cite{bemporad2002explicit} is a model for the environment that is widely used in control theory, where the state of the system evolves according to a linear dynamical system and the reward is quadratic. For the online control of this problem, when the system matrix is unknown, regret bound of order $\widetilde {\mathcal O}(\sqrt K)$ were obtained by \citet{abbasi2011regret} for a computationally inefficient algorithm. This limitation was then removed by \citet{dean2018regret, cohen2019learning}.
\emph{Linear MDPs}~\citep{yang2019sample,jin2020provably} is a widespread setting in RL theory where another form of linearity is assumed. Different from LQRs, here the transition kernel of the MDP can be factorized as a scalar product between a feature map $\bm \varphi: \Ss \times \As \to \R^d$ and an unknown vector of finite measures over $\Ss$. The reward function is typically assumed to be linear in the same features. When the feature map is known, regret bounds of order $\widetilde {\mathcal O}(\sqrt {d^3K})$ are possible \cite{jin2020provably}. Still, these are two examples of \emph{parametric} settings, which do not constitute a reasonable assumption for general continuous-space MDPs.
A much wider family can be defined by just assuming that small variations in the state-action pair $(s,a)$ lead to ($i$) small variations in the reward function $r(s,a)$ ($ii$) small variations in the transition function $p(\cdot|s,a)$, as it is assumed in the setting of \emph{Lipschitz MDPs}~\citep{rachelson2010locality}. Lipschitz MDPs have been applied to a number of different settings. Not only do they allow developing theoretically grounded algorithms \citep{pirotta2015policy, asadi2018lipschitz, metelli2020control}, they also help to tackle generalizations of standard RL, such as RL with delayed feedback~\citep{liotet2022delayed} or configurable RL~\citep{metelli2022exploiting}, and auxiliary tasks for imitation learning~\citep{damiani2022balancing, maran2023tight}. The price of being very general is paid with a regret bound that is much worse than that of previous families. Indeed, no algorithm can achieve a better regret bound than $\Omega(K^{\frac{d+1}{d+2}})$, in terms of dependence on $K$, being $d$ the dimensionality of the state-action space. This entails a huge performance detriment compared with Linear MDPs and LQRs, where the order of the regret in $K$ is $\frac{1}{2}$, regardless of the dimension $d$. In fact, there is still a large gap in the theory between parametric families of MDPs and Lipschitz MDPs, and little is known about what lies in between.
One last family of continuous-state MDPs for which regret bounds exist is that of \emph{Kernelized MDPs}~\citep{yang2020provably}, where both the reward function and the transition function belong to a \emph{reproducing kernel Hilbert space} (RKHS) induced by a known kernel. In the typical application to continuous-state MDPs, the kernel is assumed to come from the Matérn covariance function with parameter $m>0$. The higher the value of $m$, the more stringent the assumption, as the corresponding RHKS contains fewer functions. Coherently, regret bounds for this setting decrease with $m$. In particular, it was very recently proved \cite{vakili2023kernelized} that an algorithm achieves regret $\widetilde {\mathcal O}(K^{\frac{m+d}{2m+d}})$ in this setting, approaching $\widetilde {\mathcal O}(\sqrt K)$ as $m \to \infty$. 
In this paper, we aim to make one first step towards reaching an analogous result in the general case of any MDP endowed with some \quotes{smoothness} property.

\textbf{Why Smoothness?}~~The presence of mathematically elegant, smooth functions in real-world phenomena of the most diverse nature has always been a source of fascination and philosophical research \cite{wigner1990unreasonable}.
Smooth functions, or even infinitely differentiable functions, play a crucial role in various scientific and engineering disciplines due to their versatility and analytical tractability. They are valuable tools for modeling complex phenomena and solving mathematical problems. In physics, for instance, smooth functions are widely employed to describe the behavior of physical systems, such as in the context of quantum mechanics and electromagnetic field theory \citep{shankar2012principles,born2013principles}. In engineering, the utility of smooth functions is evident in control systems and signal processing, where they simplify the analysis and design of dynamic systems \citep{oppenheim1997signals, ogata2010modern}. Smooth functions are not just a formalism but a fundamental and practical mathematical framework for understanding and manipulating real-world phenomena. The reason why these functions are ubiquitous in the natural sciences can be attributed to their connection with partial differential equations. Many natural phenomena can be described by a limited number of partial differential equations, which have the characteristic of enforcing strong regularity conditions on solutions. In particular, thermal, electromagnetic, and wave phenomena are governed by three well-known different equations: the heat equation, the Laplace-Poisson equation, and the D'Alembert equation, respectively \citep{sobolev1964partial,tikhonov2013equations,salsa2022partial}. Each of these is characterized by inherent regularity properties so that solutions are infinitely differentiable under suitable boundary conditions.

\textbf{Our Contributions}~~In this paper, we introduce two very general classes of MDP based on the notion of $\nu$-smoothness either applied to the transition model and to the reward function (\emph{Strongly Smooth MDP}) or to the Bellman operator (\emph{Weakly Smooth MDP}) (Section \ref{sec:smoothnessmdp}). We develop a novel technique that builds upon results from the theory of \emph{orthogonal functions} (specifically, Legendre polynomials) to design algorithms, \textsc{Legendre-LSVI} and \textsc{Legendre-Eleanor}, characterized by different computational costs, for addressing regret minimization in smooth MDPs (Section \ref{sec:orthogonal}). Then, we provide the theoretical analysis of the proposed algorithms showing that, under appropriate conditions on smoothness constant $\nu$, they fulfill the no-regret property with a regret rate depending on $\nu$ (Section \ref{sec:theo}). Finally, to compare our results with the state-of-the-art theoretical RL, we show that ($i$) our setting includes the most common classes of problems for which no-regret guarantees have been shown ($ii$) general-purpose RL algorithms that apply to our setting obtain worse regret guarantees than ours (Section \ref{sec:related}).
The proofs of all the results presented in the main paper are reported in Appendix~\ref{apx:proofs}.

\section{Preliminaries}

\textbf{Markov decision processes and policies}~~We consider a finite-horizon Markov decision process (MDP)~\cite{puterman2014markov} $M=(\Ss, \As, p, r, H)$, where $\Ss = [-1,1]^{d_S}$ is the state space, $\As = [-1,1]^{d_A}$ is the action space,\footnote{Choosing these compacts set is without loss of generality as, provided a suitable rescaling, any compact set could be used.} $p=\{p_h\}_{h=1}^{H-1}$ is the sequence of transition functions, each mapping a pair $(s,a) \in \mathcal{S \times A}$ to a probability distribution $p_h(\cdot |s,a)$ over $\Ss$, while the initial state $s_1$ may be chosen arbitrarily from the environment; $r=\{r_h\}_{h=1}^H$ is the sequence of reward functions, each mapping a pair $(s,a)$ to a real number $r_h(s,a)$, and $H$ is the time horizon.
At each episode $k\in [K] \coloneqq \{1,\dots, K\}$, the agent chooses a policy $\pi_k = \{\pi_{k,h}\}_{h=1}^H$, which is a sequence of mappings from $\Ss$ to the probability distributions over $\As$. For each stage  $h \in [H]$, the action is chosen according to $a_h\sim \pi_{k,h}(\cdot |s_h)$,  the agent gains reward $r_h(s_h,a_h)+\eta_h$, where $\eta_h$ is a $\sigma-$subgaussian noise independent of the past, and the environment transitions to the next state $s_{h+1}\sim p_h(\cdot |s_h,a_h).$ In this setting, it is useful to define the following quantities.

\textbf{Value functions and Bellman operators.}~~
The state-action value function (or \emph{Q-function}) quantifies the expected sum of the rewards obtained under a policy $\pi$, starting from a state-stage pair $(s,h)\in\Ss\times [H]$ and fixing the first action to some $a\in\As$:
\begin{align}\label{eq:action_state_val}
    Q_h^{\pi}(s,a) \coloneqq \mathbb{E}_{\pi} \left[ \sum_{\ell=h}^{H} r_\ell(s_\ell,a_\ell)\bigg | s_0=s,a_0=a  \right],
\end{align}
where $\E_{\pi}$ denotes expectation w.r.t. to the stochastic process $a_h \sim \pi_h(\cdot|s_h)$ and $s_{h+1} \sim p_h(\cdot|s_h,a_h)$ for all $h \in [H]$.
The state value function (or \emph{V-function}) is defined as $V_h^\pi(s)\coloneqq\E_{a\sim\pi_h(\cdot\vert s)}[Q_h^\pi(s,a)]$, for all $s\in\Ss$. The supremum of the value functions between all the policies take the name of optimal value functions, and are written as $Q_h^*(s,a):=\sup_\pi Q_h^\pi(s,a),\ V_h^*(s):=\sup_{\pi}V_h^\pi(s)$.

In this work, as often done in the literature, we assume that the reward is normalized in a way that $|Q_h^\pi(s,a)|\le 1$ for every $s\in \Ss, a\in A$ and $ h\in [H]$.\footnote{Sometimes it is instead assumed that the reward lies in $[-1,1]$ so that the total return is bounded in $[-H,H]$.}
The evaluation of the expected return is linked to the notion of Bellman operators. For a policy $\pi$, the corresponding \emph{Bellman operator} $\mathcal T^\pi$ is defined as follows, for every $h\in [H]$ and every function $f:\Ss\times \As \to \R$:
$$\mathcal T^\pi_h f(s,a):=r_h(s,a)+\E_{\substack{s'\sim p_h(\cdot|s,a)\\a'\sim \pi_h(\cdot|s)}}[f(s',a')].$$
Even more crucial for control is the \textit{Bellman optimality operator}, which, instead of fixing the policy, chooses the maximum of $f$ for the next state:
$$\mathcal T^*_h f(s,a):=r_h(s,a)+\E_{s'\sim p_h(\cdot|s,a)}\Big[\sup_{a'\in \As}f(s',a')\Big].$$

\textbf{Agent's goal.}~~The agent aims at choosing a sequence of policies $\pi_k$ in order to minimize the cumulative difference between the expected return of their policies $J^{\pi_k}$ and the optimal one, given the initial state chosen by the environment. This quantity takes the name of \textit{(cumulative) regret}:
$$R_K\coloneqq\sum_{k=1}^K \left(V_1^*(s_1^k) - V_1^{\pi_k}(s_1^k) \right).$$
Note that if $R_K=o(K)$ for any $K$ with some probability, then, with the same probability, $J^{\pi_k} \to \sup_{\pi} J^\pi$ as $K\to\infty$. An algorithm choosing a sequence of policies with this property is called \emph{no-regret}.

\textbf{Smoothness of real functions.}~~
Let $\Omega \subset [-1,1]^d$ and $f : \Omega \to \R$. We say that $f\in \mathcal C^{\nu,1}(\Omega)$ if there exists a constant $L<+\infty$ such that $f$ is $\nu-$differentiable (i.e., differentiable $\nu$ times),
and for every multi-index $\bm{\alpha}=(\alpha_1,\dots,\alpha_d)$ with $|\bm \alpha| \coloneqq \sum_{i=1}^d \alpha_i \le \nu$ we have:
\begin{equation}\forall x,y \in \Omega: \quad |D^{\bm \alpha}f(x)-D^{\bm \alpha}f(y)|\le L\|x-y\|_2,\label{eq:lipdef}\end{equation}
where the multi-index derivative is defined as follows
$D^{\bm \alpha}f:=\frac{\partial^{ \alpha_1+...+  \alpha_d}}{\partial x_1^{ \alpha_1}\dots \partial x_d^{ \alpha_d}}f.$
The set $\mathcal C^{\nu,1}(\Omega)$ forms, for every value of $\nu$, a normed vector space, and more precisely, a Banach space \cite{kolmogorov1975introductory}. A norm for which this holds is given by $\|f\|_{\mathcal C^{\nu,1}}:=\max_{| \bm\alpha|\le \nu+1}\|D^{\bm\alpha} f\|_{L_\infty}$.
This definition may seem counter-intuitive since derivatives up to order $\nu+1$ appear, but in fact, this is because the derivative of a Lipschitz function is defined almost everywhere, and its $L_\infty-$norm equals the Lipschitz constant itself \citep{rudin1974real}.
The most straightforward case of this definition is given by $\mathcal C^{0,1}(\Omega)$, corresponding to the space of Lipschitz continuous functions, where the semi-norm $\|\cdot\|_{\mathcal C^{0,1}}$ corresponds exactly to the \textit{Lipschitz constant} of a function. The concept of Wasserstein metric $\wass(\cdot,\cdot)$, a notion of distance for probability measures on metric spaces, is strictly related to Lipschitz functions. For two measures $\mu,\zeta$ on a metric space $\Omega$, this distance is defined as $\wass(\mu,\zeta):=\sup_{f\in \mathcal C^{0,1}: \|f\|_{\mathcal C^{0,1}}=1} \int_\Omega f(\omega) d(\mu-\zeta)(\omega)$.
We will also make use of the space $\mathcal C^{\infty}(\Omega):=\cap_{\nu=1}^\infty \mathcal C^{\nu,1}(\Omega)$ of indefinitely differentiable functions. Despite assuming a function is $\mathcal C^{\infty}(\Omega)$ seems restrictive, this class includes polynomial, trigonometric, and exponential functions.

\section{Smoothness in MDPs}\label{sec:smoothnessmdp}

In this section, we introduce two sets of assumptions concerning the MDP. We start with the simpler one, which we call \textit{Strongly Smooth} MDP. This assumption is similar to the kernelized MDP setting~\citep{yang2020provably}, as it bounds the norm of the transition function and the reward function in a given space, but without specifying an explicit structure. Instead of assuming that they belong to a given RKHS, we limit our assumption to their smoothness.
\begin{ass}\textit{(Strongly Smooth MDP)}. An MDP is a \emph{Strongly Smooth} of order $\nu$ if:
    $$\forall h\in [H]\ \forall s'\in \Ss,\quad r_h(\cdot,\cdot),p_h(s'|\cdot,\cdot)\in \mathcal C^{\nu,1}(\Ss \times \As),$$
    with $\sup_{h,s'} \|p_h(s'|\cdot,\cdot)\|_{\mathcal C^{\nu,1}}, \sup_{h,s'} \|r_h(\cdot,\cdot)\|_{\mathcal C^{\nu,1}}<+\infty$.
\end{ass}

Note that assuming the finiteness of the norms but not the knowledge of its upper bound, is different from what is asked in the analogous assumption for kernelized MDPs. 

Assuming this form of regularity of the reward function seems fair. Being very often a human-designed function, we can expect it to be indefinitely differentiable most of the time. For what concerns the transition function, this requirement is more tricky. Indeed, nontrivial transition functions for deterministic MDPs often take the form $p(s'|s,a)=\delta(s'=f(s,a))$, for some function $f: \Ss \times \As \to \Ss$. This function does not satisfy Strong Smoothness, even when $f$ is itself very smooth, as the Dirac delta $\delta(\cdot)$ is not a continuous function. 
For this reason, we introduce a more general assumption, which directly concerns the Bellman optimality operator.
\begin{ass}\textit{(Weakly Smooth MDP)}. An MDPs is \emph{Weakly Smooth} of order $\nu$ if, for every $h\in [H]$ the Bellman optimality operator $\mathcal{T}^*_h$ is bounded on $\mathcal C^{\nu,1}(\Ss \times \As) \to \mathcal C^{\nu,1}(\Ss \times \As)$.
\end{ass}

Boundedness over $\mathcal C^{\nu,1}(\Ss \times \As) \to \mathcal C^{\nu,1}(\Ss \times \As)$ means that the operator cannot output a function that is not $\mathcal C^{\nu,1}$ when receiving a function from the same set. Moreover, there exists a constant $C_{\mathcal T^*} < +\infty$ such that $\|\mathcal T^*_h f\|_{\mathcal C^{\nu,1}} \le C_{\mathcal T^*} (\|f\|_{\mathcal C^{\nu,1}}+1)$ for every $h\in [H]$ and every function $f\in \mathcal C^{\nu,1}(\Ss \times \As)$. In Appendix \ref{app:smoothvssmooth}, we show that Weak Smoothness is (much) weaker than  Strong Smoothness.

\section{Orthogonal Function Representations}\label{sec:orthogonal}

Our approach to the solution of Strongly and Weakly Smooth MDPs is based on the idea of finding a representation of the state-action space $\Ss \times \As$ such that the problem is reduced to a Linear MDP in a feature space. To achieve this result, we will use a particular class of feature maps based on \emph{Legendre polynomials}~\citep{quarteroni2010numerical}.
\begin{definition}[Legendre feature map]
    Let $\varphi_{L,n}(x)$
    be the $n$-th order Legendre polynomial, we define, for every $N\in \N$, the feature map $\boldsymbol{\varphi}_{L,N}: [-1,1]\to \mathbb R^N$ as follows:
    $$\boldsymbol{\varphi}_{L,N}(x):= N^{-1/2}(\varphi_{L,0}(x), \dots ,\varphi_{L,N}(x)).$$
\end{definition}
The importance of these feature maps, not apparent from their definition, lies in their \emph{orthogonality}. In fact, Legendre polynomials are such that $\int_{-1}^1 \varphi_{L,i}(x)\varphi_{L,j}(x)\ \mathrm{d}x=\delta_{ij}$, which is $1$ if $i=j$, $0$ otherwise.\footnote{We use Legendre polynomials which are normalized in the space $L^2$, while some authors normalize them differently.}
The multidimensional generalization of this map to $[-1,1]^d$, which preserves the orthogonality property, is obtained by a Cartesian product operation. Precisely, we call the generalization of the Legendre map to $[-1,1]^d$ as $\boldsymbol \widetilde N^{-1/2}\bm \varphi_{L,N}^d(x_1,\dots, x_d)$, where $\bm \varphi_{L,N}^d(x_1,\dots, x_d)$ stacks, in its $\widetilde N$ components, all possible products of Legendre polynomials in the variables $x_1,\dots, x_d$ such that the total degree (sum of the degrees of the single polynomials) does not exceed $N$ 
(and $\widetilde N^{-1/2}$ is just a normalization term).
A formal definition of the feature map is given in the following:
\begin{align}
& \mathcal{L}_N = \{(g_1,\dots,g_d) \in \{0,\dots,N\}^d \,:\, \sum_{i=1}^d g_i \le N\}\\
    &\bm \varphi_{L,N}^d(x_1,\dots, x_d) =\left(\prod_{i=1}^d\varphi_{L,g_i}(x_i)\right)_{(g_1,\dots,g_d) \in \mathcal{L}_N}
\end{align}
This definition draws an analogy with Fourier series \cite{katznelson2004introduction}, which are built on another family of orthogonal functions. As for the Fourier series, we can use a linear combination of Legendre polynomials to approximate any smooth function. However, while the convergence of the Fourier series is only guaranteed for periodic functions, Legendre polynomials are not affected by this limitation. 

\subsection{Weakly Smooth MDPs: \textsc{Legendre-Eleanor}}\label{sec:leg-ele} We start from the most general case, i.e., that of Weakly Smooth MDPs. We can show that the pair given by an MDP of this class and our  Legendre feature map forms a process with \textit{low inherent Bellman error}~\citep{zanette2020learning}. 
Given any feature map $\bm \varphi: \Ss \times \As \to \R^d$ and a sequence of compact sets $\mathcal B_{h}\subset \R^d$ for $h\in [H]$, calling $Q_{\theta}(s,a)$ the function $\bm \varphi(s,a)^\top \theta$, the inherent Bellman error w.r.t. $\{\mathcal B_{h}\}_h$ is defined as:
\begin{equation}\mathcal I:=\max_{h\in [H]}\sup_{\theta\in \mathcal B_{h+1}}\inf_{\theta'\in \mathcal B_{h}}\|\bm \varphi(s,a)^\top \theta'-\mathcal T^*Q_{\theta}(s,a)\|_\infty.\label{eq:inherent}\end{equation}
This definition illustrates, intuitively, that starting from a $Q$-function in the span of $\bm \varphi$, the Bellman optimality operator produces another one which is $\mathcal I$-close to the span of $\bm \varphi$. We can prove that a Weakly Smooth MDP equipped with a Legendre feature map $\bm \varphi_{L,N}^d$ has bounded inherent Bellman error $\mathcal I$, where the bound depends on the order of smoothness $\nu$ (Theorem \ref{thm:IBE} in the appendix). Using the Legendre representation along with \textsc{Eleanor}~\cite{zanette2020learning}, an algorithm designed for MDPs with low inherent Bellman error, we can achieve the no-regret property under mild assumptions.
We call the resulting algorithm \textsc{Legendre-Eleanor}, and we will theoretically analyze it in Section \ref{sec:theo}.

\subsection{Strongly Smooth MDPs: \textsc{Legendre-LSVI}}\label{sec:compcomp}

In the previous section, we have presented an approach to solving Weakly Smooth MDPs. Still, the algorithm that we introduced is based on \textsc{Eleanor}, which is known to be computationally inefficient~\cite{zanette2020learning}. A major challenge in RL is to devise algorithms that are both no-regret and have a running time polynomial in the task horizon, problem dimension, and number of episodes. This motivates the search for a \emph{polynomial-time} algorithm that can achieve no regret under the Strongly Smooth assumption. This is possible thanks to the fact that, when we apply the Legendre representation of a Strongly Smooth MDP, we get not only an MDP with low inherent Bellmann error, but a Linear MDP, for which computationally feasible algorithms, such as \textsc{LSVI-UCB} \cite{jin2020provably}, are known.

We call the resulting algorithm \textsc{Legendre-LSVI}. To analyze its computational complexity, remembering that we have called $\widetilde N=\binom{N+d}{N}\le N^d$ the dimension of the feature map used, we can just replace this value in the computational complexity of LSVI-UCB. 
As it is well-known that the time complexity of LSVI-UCB is $\mathcal O(K^2H+\widetilde N^3KH)$, \textsc{Legendre-LSVI} has polynomial time complexity, provided that we choose $N$ so that term $\widetilde N$ is not exponential in the relevant quantities.

\begin{figure}{}
     \centering
     \begin{subfigure}[b]{0.24\textwidth}
         \centering
         \includegraphics[scale=.5]{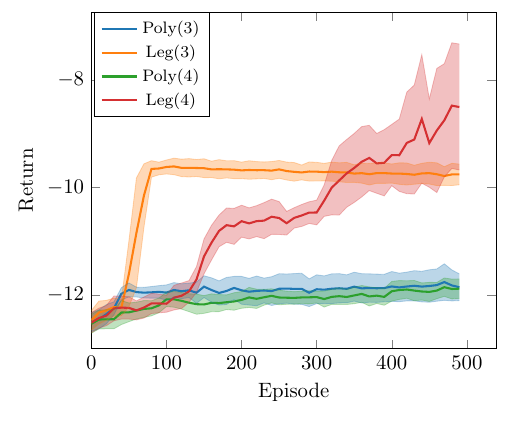}
          \label{fig:1d:gaussian}
     \end{subfigure}%
     \hfill
     \begin{subfigure}[b]{0.24\textwidth}
         \centering
         \includegraphics[scale=.5]{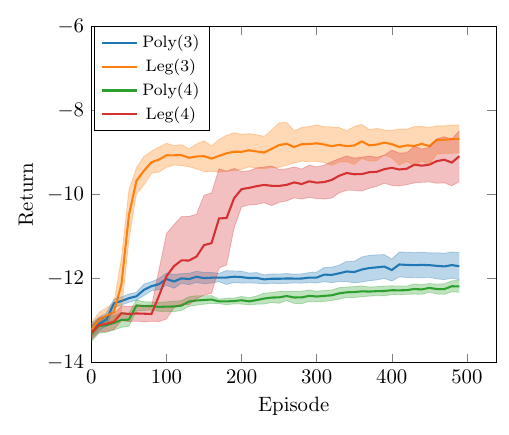}
          \label{fig:1d:poly}
     \end{subfigure}%
    \caption{Curve of the episodic return for the simulation in Section \ref{sec:expe} with 95\% confidence intervals over five random seeds.}
    \label{fig:pqr}
\end{figure}

\newcommand{\resultrow}[6]{ #1 #3 #2 #6 #4 #5 }
\begin{table*}[t]
    \centering
  \begin{tabular}{lllllll}
    \toprule
    Algorithm  \resultrow{& \bfseries W.ly Smooth}{& \bfseries S.ly Smooth}{&Lipschitz}{&LQR}{&LinearMDP}{&Kernelized} \\
    \midrule
    \rowcolor{vibrantGrey!10}\textsc{\textsc{Legendre-Eleanor}} (\ref{sec:leg-ele}) \resultrow{&$\green{K^{\frac{3d/2+\nu+1}{d+2(\nu+1)}}}$ }{ &$\green{K^{\frac{3d/2+\nu+1}{d+2(\nu+1)}}}$ }{ &$K^{\frac{3d/2+1}{d+2}}$ }{ &$\green{K^\frac{1}{2}}$ }{ &$\green{K^\frac{1}{2}}$ }{ &$K^{\frac{3d/2+\lceil m \rceil}{d+2\lceil m \rceil}}$}\\
    \cite{jin2021bellman} \resultrow{& $K^{\frac{2\nu+3d+2}{4\nu+4}}$}{ & $K^{\frac{2\nu+3d+2}{4\nu+4}}$}{ & $K^{\frac{3d+2}{4}}$}{ & $\green{K^\frac{1}{2}}$ }{& $\green{K^\frac{1}{2}}$ }{& $K^{\frac{2\lceil m \rceil+3d}{\lceil m \rceil}}$}\\
    \cite{song2019efficient}  \resultrow{&\xmark }{ &$K^{\frac{d+1}{d+2}}$ }{ &\green{$K^{\frac{d+1}{d+2}}$} }{ &$K^{\frac{d+1}{d+2}}$ }{ &$K^{\frac{d+1}{d+2}}$ }{ &$K^{\frac{d+1}{d+2}}$}\\
    \rowcolor{vibrantGrey!10}\textsc{\textsc{Legendre-LSVI}} (\ref{sec:compcomp}) \resultrow{& \xmark}{ & $K^{\frac{2d+\nu+1}{d+2(\nu+1)}}$ }{& \xmark }{& $\green{K^\frac{1}{2}}$ }{& $\green{K^\frac{1}{2}}$ }{& $K^{\frac{2d+\lceil m \rceil}{d+2\lceil m \rceil}}$}\\
    \citep{vakili2023kernelized}  \resultrow{& \xmark}{ & \xmark}{ & \xmark}{ & \xmark}{ & $\green{K^\frac{1}{2}}$}{ & $\green{K^{\frac{d+m+1}{d+2(m+1)}}}$}\\
    \cite{dean2018regret}  \resultrow{& \xmark }{& \xmark }{& \xmark }{& $\green{K^\frac{1}{2}}$ }{& \xmark  }{& \xmark}\\
    \cite{jin2020provably}  \resultrow{& \xmark}{ & \xmark}{ & \xmark}{& \xmark }{& $\green{K^\frac{1}{2}}$ }{& \xmark}\\
    \bottomrule
  \end{tabular}  
    \caption{\label{tab:sota2}Table containing the regret guarantee of each algorithm presented in the main paper for each setting. For convenience, we recall that \citet{song2019efficient,vakili2023kernelized,dean2018regret,jin2020provably} represent the state of the art for Lipschitz MDPs, Kernelized MDPs, LQRs and LinearMDPs respectively. On the columns we have the algorithm, and on the rows the setting in which it is tested. Some specifications are needed: 1) we have only reported the order of the regret in $K$, ignoring the other parameters of the problems and the logarithmic terms 2) for the linear MDP setting we have assumed that the feature map is indefinitely differentiable, an assumption explained in detail in the corresponding section of the main paper 3) for the linear MDP, the algorithms SOTA-Linear and SOTA-Kern assume to know the feature map, while our algorithm do not have this requirement 4) Kernelized MDP assume Mat\'ern kernel of order $m$. As a result of the table, we can see that our algorithm \textsc{Legendre-Eleanor} is the best performing between the ones having no-regret guarantees for all the settings: the only algorithms that are able to surpass its performance are designed for settings that are much more specific. }
\end{table*}

\subsection{Why Orthogonal Features?}\label{sec:expe}
The reader may wonder how crucial is the choice of an \emph{orthogonal} feature representation. Before moving to the theoretical analysis that will analytically justify this choice, in this subsection, we empirically show, on an illustrative problem, that the use of orthogonal features has beneficial effects on learning performance. We employ two modified versions of the LQR, in which the state, after the linear dynamic transition, is pushed towards the origin in a way that prevents it from escaping from a given compact set. Precisely, using the same formalism of the LQR, we have:
$s_{h+1} = g(As_h + Ba_h + \xi_h),$
$r_{h} = -s_h^\top Qs_h - a_h^\top Ra_h,$
where $g(x):=\frac{x}{1+\|x\|_2}$ and $\xi_h$ is a Gaussian noise. As the support of the Gaussian distribution is the full $\R^d$, after applying $g(\cdot)$, the possible set of new states is the ball of radius one. We performed two experiments with different parameter values and with horizon $H=20$, whose details can be found in the appendix \ref{app:num}.
In Figure \ref{fig:pqr}, we can see plots showing the episodic return of the algorithms as a function of the number of learning episodes. As a learning algorithm, we can see two ``correct'' versions of \textsc{Legendre-LSVI}, which are called Leg(3) and Leg(4), against two ``na\"ive'' versions of the same algorithm, Poly(3) and Poly(4). For all the algorithms, the number between the brackets corresponds to the degree of the polynomials used so that the approximation order and the length of the feature vector are equal in the two cases. The difference between the Poly() and the corresponding Leg() algorithm lies in the fact that the former is the standard basis of multivariate polynomials (e.g., $\{1,x,y,x^2,y^2,xy,\dots\}$), while the latter corresponds to the (orthogonal) Legendre basis, for which theoretical guarantees hold. Using standard polynomials as feature maps is common in practice. However, the results show that baselines using Legendre polynomials achieve much superior episodic return compared with the analog with standard polynomials, as it is predicted from the theory behind our results. 
The latter, in green and blue, failed to achieve significant learning throughout 500 episodes in either environment. On the contrary, Leg(3), in orange, is able to learn a good policy suddenly and subsequently settles down, obtaining an almost constant return in all the following episodes. Leg(4) proves to learn more slowly than Leg(3), but in the first environment, it obtains a higher return value, while in the second environment, it obtains a comparable one. These results are consistent with the theory, as increasing the dimensionality of the feature map considered and increasing the degree of the polynomial from 3 to 4 has the effect of slowing down learning but improving the order of approximation to converge to a higher return.

\section{Theoretical Guarantees}\label{sec:theo}

In this section, we derive the regret bounds for our two algorithms \textsc{Legendre-Eleanor} and \textsc{Legendre-LSVI}. The former is able to achieve the no-regret property for Weakly Smooth MDPs under the assumption that $2\nu \ge d-2$, as shown in the following result.
\begin{thm}\label{thm:ele}
    Let us consider a Weakly Smooth MDP $M$ with state action space $[-1,1]^d$. Under the condition that $\nu>d/2-1$, \textsc{Legendre-Eleanor} initialized with $N=\lceil K^{\frac{1}{d+2(\nu+1)}} \rceil$, with probability at least $1-\delta$, suffers a regret of order at most:
    $$R_K\le \bigot \left ( C_{\text{\textsc{Ele}}}^HK^{\frac{3d/2+\nu+1}{d+2(\nu+1)}} \right),$$
    where the constant depends only on $d$ and $\nu$ and the $\bigot$ hides logarithmic functions of $K$, $\delta$.
\end{thm}
The proof is provided in Appendix \ref{app:proofele}. The fact that the regret grows exponentially in $H$ is annoying but unavoidable. Indeed, we derive a lower bound that shows that any MDP class that is rich enough to capture Lipschitz MDPs must have a regret bound which is exponential in $H$ (see Appendix \ref{app:expoH}). No surprise, all related works on Lipschtz MDPs are affected by the same problem. Apart from that, Theorem \ref{thm:ele} shows the bound we aimed for. In the ``good" regime, where $\nu>d/2-1$, we are able to prove a regret bound that is monotonically decreasing in $\nu$, and approaches $\sqrt K$ for $\nu \to \infty$. Therefore, our model can cover both general (Lipschitz MDPs and Kernelized MDPs) and specific (LQRs, Linear MDPs) models, with a regret bound that is adaptive to the higher smoothness.

We now turn to \textsc{Legendre-LSVI}. Its guarantees are restricted to Strongly Smooth MDPs, and it only achieves no regret under the more demanding requirement $\nu \ge d-1$. Still, its value lies both in its polynomial computational complexity and in its polynomial dependence on the horizon $H$.
\begin{thm}\label{thm:lsvi}  Let us consider a Strongly Smooth MDP $M$ with state action space $[-1,1]^d$. Under the condition that $d\le \nu+1$, \textsc{Legendre-LSVI} initialized with $N=\lceil K^{\frac{1}{d+2(\nu+1)}} \rceil$, with probability at least $1-\delta$, suffers a regret of order at most:    
    $$R_K \le \bigot \left(H^{3/2}K^{\frac{2d+\nu+1}{d+2(\nu+1)}}\right),$$
where $\bigot$ hides logarithmic functions of $K$, $\delta$, and $H$.
\end{thm}
The order of the regret in $K$ is worse than the one of \textsc{Legendre-Eleanor} but we still have $\sqrt{K}$ in the limit $\nu\to +\infty$ and, as anticipated, the exponential growth in $H$ is avoided.
The proof is provided in Appendix \ref{app:prooflsvi}. Note that the choice $N=\lceil K^{\frac{1}{d+2(\nu+1)}} \rceil$ makes the running time polynomial. Indeed, we have seen in Section \ref{sec:compcomp} that the latter scales as $\mathcal O(K^2H+\widetilde N^3KH)$, so that this choice of $N$ allows bound the time complexity as $\mathcal O(K^2H+K^{1+\frac{3d}{d+2(\nu+1)}}H)= \mathcal O(K^2H)$.

\section{Comparison with Related Literature}\label{sec:related}\label{sec:comparison}

Achieving regret or sample complexity guarantees for MDPs with continuous state and action spaces has been one of the main challenges of theoretical RL in recent years. Among the numerous papers dealing with this problem, we provide a brief overview of the most significant results achieved under the most common assumptions proposed in the literature. This way, we show how very different problems studied so far are included in our setting. The overall inclusion relationships between the settings are summarized in Figure \ref{fig:eulerovenn}, while the relations between the regret bounds in are shown in table \ref{tab:sota2}.

\subsection{Lipschitz MDPs} 
Lipschitz MDPs assume that the transition function of the model, as well as the reward function, are Lipschitz continuous with constants $L_p$ and $L_r$, respectively. For the reward function, this is just expressed by enforcing for all $ h\in [H]\ s,s'\in \Ss,\ a,a'\in \As$:
$$|r(s,a)-r(s',a')|\le L_r(\|s-s'\|_2+\|a-a'\|_2).$$
While, as the transition function maps a state-action pair into a probability distribution, we need a metric for probability distribution to define Lipschitzness. This is commonly done by means of the Wasserstein metric $\wass(\cdot,\cdot)$~\cite{rachelson2010locality}:
$$\wass(p_h(\cdot|s,a),p_h(\cdot|s',a'))\le L_p(\|s-s'\|_2+\|a-a'\|_2).$$
Learning in Lipschitz MDPs is a very hot topic \citep{ortner2012online, sinclair2019adaptive, song2019efficient, sinclair2020adaptive, domingues2020regret, le2021metrics} and many regret bounds of order $K^{\frac{d+1}{d+2}}$ have been proved.

\textbf{Lipschitz MDPs are Weakly Smooth but not Strongly Smooth.}~~Lipschitz MDPs represent the most general class of continuous space MDPs studied in the literature. These processes are not necessarily Strongly Smooth, as they include deterministic processes that cannot be Strongly Smooth due to the discrete nature of the transition function $p$. Still, it can be proved that Lipschitz MDPs are Weakly Smooth for $\nu=0$, as shown in the Appendix \ref{app:lipMDP}.

\subsection{Linear Quadratic Regulator} 
Linear Quadratic Regulator (LQR) is a model of the environment s where the state and action space are $\Ss=\R^{d_S},\As=\R^{d_A}$, and there exist two matrices $A\in \R^{d_S\times d_S}$ and $ B\in \R^{d_S\times d_A}$ defining the transition model $s_{h+1} = As_h + Ba_h + \xi_h,$ where $\xi_h \sim \mathcal{N}(0,\Sigma)$ is a Gaussian noise. Also, there are other two positive semi-definite matrices $Q\in \R^{d_S\times d_S}$ and $ R\in \R^{d_S\times d_S}$ such that the reward function is given by $r_{h} = -s_h^\top Qs_h - a_h^\top Ra_h.$
Regret bounds of order $\sqrt{K}$ were obtained for this kind of MDPs~\citet{abbasi2011regret,dean2018regret, cohen2019learning}. This result has been generalized by \citet{kakade2020information}, preserving the optimal $\sqrt{K}$ regret order when the linear dynamics are composed with a known feature map. Note that, differently from the previous case, the dimension $d$ of the state-action space does not affect the regret order, as its dependence in the regret takes the form $\text{poly}(d_\Ss,d_\As)\sqrt{K}$.

\textbf{LQRs are Strongly Smooth for $\nu = +\infty$.}~~The class of LQRs contains only processes which are indefinitely differentiable. Indeed, the reward function is quadratic, while the transition one can be written as $p_h(s'|s,a)=\mathcal N(s'; As_h + Ba_h, \Sigma)$. The last function is $\propto \exp \left (-(s'-As_h-Ba_h)^\top \Sigma^{-1}(s'-As_h-Ba_h)\right)$, which is indefinitely differentiable is all its variables. For $\nu = +\infty$ our Theorem \ref{thm:lsvi} ensures that regret of order $\sqrt{K}$ can be achieved, which is coherent with the results from the literature.

\begin{figure}[t]{}
     \centering
     \includegraphics[scale=.6]{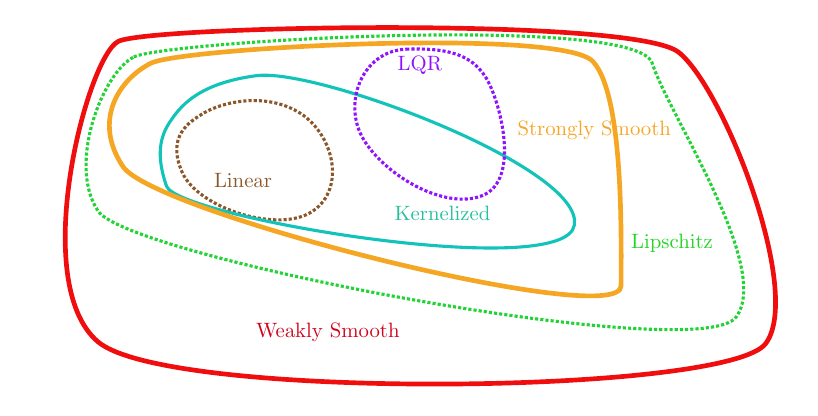}
    \caption{A schematic summarizing relations among families of continuous space RL problems. Our assumptions correspond to the red and orange sets.}
    \label{fig:eulerovenn}
\end{figure}

\subsection{Linear MDPs}\label{par:linearMDP} 
Linear MDPs are processes satisfying, for every $ h\in [H]\ s,s'\in \Ss,\ a\in \As$: 
$p_h(s'|s,a)=\langle \bm \mu_h(s'), \bm \varphi(s,a)\rangle,$
where $\bm \varphi(s,a)$ (resp. $\bm \mu(s')$) are fixed functions from $\Ss \times \As$ (resp. $\Ss$) to $\R^{d_{\bm \varphi}}$. Similarly, the reward function factorizes as $r_h(s,a)=\langle \bm \theta_h, \bm \varphi(s,a)\rangle,$
for some vector $\bm 
 \theta_h\in\R^{d_{\bm \varphi}}$. 
For linear MDPs, if the feature map is given, the optimal regret order of $\sqrt{K}$ can be achieved. For example, in \citep{jin2020provably} the regret takes the form of $d_{\bm \varphi}^{3/2}K^{1/2}$, so that the order of the regret is not affected by the magnitude of $d$ or $d_{\bm \varphi}$.
When the feature map is not known in advance, the problem becomes significantly harder~\citep{agarwal2020flambe, uehara2021representation}. In fact, there is currently no work able to prove regret guarantees for this setting, although sample complexity guarantees are available.

\textbf{Linear MDPs are Strongly Smooth.}~~Linear MDPs are Strongly Smooth with an order $\nu$ depending on the smoothness of the feature map $\bm  \varphi$ (see Appendix \ref{app:linMDP}). If this function is handcrafted, such as a polynomial, exponential, or trigonometric function, we have $\nu=+\infty$. The most favorable case also happens when $\bm \varphi$ is a fully connected neural network with either sigmoid \cite{narayan1997generalized}, tanh \cite{abdelouahab2017tanh}, or softplus \cite{zheng2015improving} activation function, as these activations are all infinitely differentiable. Instead, using ReLU \cite{schmidt2020nonparametric} activation, which is only Lipschitz continuous, leads to $\nu=0$. Applying our theorem \label{thm:lsvi} ensures that when using a feature map that is indefinitely smooth, \textsc{Legendre-LSVI} has a regret order of $\sqrt{K}$, \textbf{\emph{even if the feature map is not explicitly known}}. This is a very strong result that has independent interest for the literature of linear MDPs.

\subsection{Kernelized MDPs}
Kernelized MDPs are built on a completely different assumption with respect to the previous methods. In fact, they start from a given kernel $k(\cdot, \cdot)$ and assume that both the transition function and the reward function belong to the RKHS $\mathcal H_k$ corresponding to the kernel $k$: 
$$\forall h\in [H],\ \forall s'\in \Ss: \qquad r_h(\cdot,\cdot),p_h(s'|\cdot,\cdot)\in \mathcal H_k.$$
Moreover, as $\mathcal H_k$ is endowed with a norm $\|\cdot\|_{\mathcal H_k}$, an upper bound on the norm of the reward and of the transition function is assumed to be known. The most common family of kernels is given by the Matérn kernels, which depend on a parameter $m>0$. Interestingly, the corresponding $\mathcal H_k$ gets smaller the higher value of $m$, and the function there contained becomes progressively smoother. For this family of MDPs, a regret bound of order $\bigot(K^{\frac{m+d}{2m+d}})$ is known, which has been shown to be optimal for both the case $m \to 0$ and for $m \to +\infty$~\citep{vakili2023kernelized}. Other papers that dealt with this setting are \citep{chowdhury2019online, yang2020function, domingues2021kernel}. 

\paragraph{Kernelized MDPs are Strongly Smooth.}~~In fact, Kernelized MDPs are just a particular case of Strongly Smooth MDPs, as directly follows from the fact that functions in the most studied RKHS are smooth. A general result (Theorem 10.45 from \citet{wendland2004scattered}) shows that whenever we have a $\nu-$times differentiable kernel $k$, the corresponding RKHS contains only functions that are $\nu/2-$times differentiable, and thus contained in $\mathcal C^{\nu/2-1,1}(\Omega)$. The result can be specialized to the Matérn family of Kernels. We show in Appendix \ref{app:kern} that the Matérn kernel of parameter $m$ generates an RKHS which is contained in $\mathcal C^{\nu-1,1}(\Omega)$ for every $\nu<m$, provided that the domain $\Omega$ satisfies some reasonable assumption there specified.

\subsection{Other Structural Assumptions: Comparison with General-Purpose Algorithms}\label{sec:jin}

In recent years, RL theory has seen a rush in searching for the weakest assumptions under which sample-efficient RL is possible. These assumptions typically generalize linear MDPs. In particular, we cite MDPs with low inherent Bellman error~\cite{zanette2020learning}, MDPs with Gaussian noise~\citep{ren2022free}, and MDPs with low Bellman-Eluder dimension~\citep{jin2021bellman}. Another strong structural assumption is introduced in \cite{du2021bilinear}, but no regret bounds are known for problems of that family. 

Among the most general algorithms with regret guarantees for general RL, there is Algorithm 1 from~\citep{ren2022free}. The latter is built on the idea that if an MDP has a transition function given by deterministic dynamics plus a Gaussian noise (like in LQRs, but without the linear structure), we can exploit the properties of the Gaussian function to our advantage. 
Applying this algorithm in the setting of Strongly Smooth MDPs, with the additional assumption that the noise is Gaussian, we can prove what follows.
\begin{thm}\label{thm:freelunch}
    The regret of Algorithm 1 from \cite{ren2022free} in a Strongly Smooth MDP, supposing that the transition function is given by a deterministic function plus a Gaussian noise, is bounded, with probability at least $1-\delta$, by:
    $$R_K\le \bigot \left(H^{\frac{3\nu+d+3}{2\nu+2}}K^{\frac{\nu+d+1}{2\nu+2}}\right),$$
    assuming that $d<\nu+1$.
\end{thm}
For the proof, see the Appendix \ref{app:proofjin}. The regret bound is similar to the one of our \textsc{Legendre-LSVI} from Theorem \ref{thm:lsvi}. Both ensure no regret if $d<\nu+1$ and are polynomial in $H$ with very similar exponents, both in $K$ and in $H$. Still, the latter result has two major drawbacks: the noise must be Gaussian, and the algorithm is \emph{not} computationally efficient, as opposed to our \textsc{Legendre-LSVI}, which runs in $\mathcal O(K^2H)$ time.

Another comparison for our algorithms is \textsc{Golf} from \cite{jin2021bellman}. The latter is guaranteed to work in a setting that is extremely general, as it only assumes the MDP to have a low Bellman-Eluder dimension. Therefore, the following result holds.

\begin{thm}\label{thm:golf}
    The regret of \textsc{Golf} on a Weakly Smooth MDP, provided that $d< \frac{2}{3}\nu+\frac{2}{3}$, is bounded, with probability at least $1-\delta$, by:
    $$R_K\le \bigot \left(C_{\text{\textsc{Golf}}}^HK^{\frac{2\nu+3d+2}{4\nu+4}}\right).$$
\end{thm}
The proof is reported in Appendix \ref{app:proofjin}. This time, no assumption is made more than the fact we are in a Weakly Smooth MDP. Therefore, the regret bound can be compared to the one of our \textsc{Legendre-Eleanor} from Theorem \ref{thm:ele}. In fact, this result only guarantees no regret under the strong assumption that $d< \frac{2}{3}\nu+\frac{2}{3}$, as opposed to Theorem \ref{thm:ele}, which only requires $d <2\nu+2$. The order in $K$ is also much better for our algorithm for every possible smoothness parameter $\nu$, while both  suffer exponential dependence on the time horizon, as it was already clear from the lower bound (see Appendix \ref{app:expoH}). Lastly, both \textsc{Golf} and \textsc{Legendre-LSVI} are computationally inefficient.



\section{Conclusions}

In this study, we defined two broad classes of MDPs distinguished by varying degrees of smoothness, which generalize most of the settings for which no-regret RL algorithms exist in the literature. Furthermore, we introduced two novel algorithms based on the theory of orthogonal functions, which are able to deal with our general setting, achieving a better regret guarantee than any previous algorithm having the same level of generality.

\textbf{Future works.}~~Despite the generality of our results, it is not clear if the proposed algorithms are optimal for our general setting, or if a better regret bound is possible. Therefore, the main objective of future work is to close this gap, by either finding a lower bound for the regret of any algorithm or by proving an improved upper bound.

\clearpage

\section*{Impact Statement}
This paper presents work whose goal is to advance the field of Machine Learning. There are many potential societal consequences of our work, none which we feel must be specifically highlighted here.

\bibliography{example_paper}
\bibliographystyle{icml2023}

\newpage
\appendix
\onecolumn
\section{Notation}

In this section, we leave, for the reader's convenience, two tables of the notations introduced in this paper. We start from one with standard RL notation:

\bgroup
\def\arraystretch{1.5}
\begin{tabular}{p{1in}p{3.25in}}
$\Ss$ & State space of an MDP\\
$\As$ & Action space of an MDP\\
$H$ & Time horizon of an MDP\\
$p_h$ & Transition function of an MDP at step $h$\\
$r_h$ & Transition function of an MDP at step $h$\\
$H$ & Time horizon of an MDP\\
$K$ & Number of interaction episodes between an MDP and a learning algorithm\\
$R_K$ & Cumulative regret after $K$ episodes\\
$\pi_h^k$ & Policy selected by the algorithm after $k$ episodes for step $h$\\
$Q_h^{\pi}(s,a)$ & State-action value function for policy $\pi$ at step $h$\\
$Q_h^{*}(s,a)$ & Optimal state-action value function at step $h$\\
$V_h^{\pi}(s,a)$ & State value function for policy $\pi$ at step $h$\\
$\mathcal T^\pi_h$ & Bellman operator for policy $\pi$ at step $h$\\
$\mathcal T^*_h$ & Bellman optimality operator at step $h$
\end{tabular}
\egroup

\bigskip
Then, we have one related to the notation coming from mathematical analysis.

\bgroup
\def\arraystretch{1.5}
\begin{tabular}{p{1in}p{5.25in}}
$d_\Ss$ & Vector space dimension of $\Ss$\\
$d_\As$ & Vector space dimension of $\As$\\
$d$ & Vector space dimension of $\Ss \times \As$\\
$D^{\bm \alpha} f$ & Derivative of function $f$ with respect to the multi index $\bm \alpha$\\
$\nu$ & order to  smoothness of a function\\
$\mathcal C^{\nu}(\Omega)$ & Space of $\nu-$times differentiable functions\\
$\mathcal C^{\nu,1}(\Omega)$ & Banach space of $\nu-$times differentiable functions with last derivative which is Lipschitz continuous\\
$\|\cdot\|_{\mathcal C^{\nu,1}}$ & Norm in the Banach space $\mathcal C^{\nu,1}(\Omega)$, so  that $\|f\|_{\mathcal C^{\nu,1}}:=\max_{|\bm \alpha|\le \nu+1}\|D^{\bm \alpha}f\|_\infty$\\
$\wass(\cdot,\cdot)$ & Wasserstein distance between measures\\
$\mathcal N(\cdot;x,\Sigma)$ & Multivariate normal distribution of mean $x$ and covariance matrix $\Sigma$\\
$I$ & Identity matrix
\end{tabular}
\egroup

\clearpage
Lastly, we have the notation which is specific for our paper and the related works.

\bgroup
\def\arraystretch{1.5}
\begin{tabular}{p{1in}p{3.25in}}
$L_p$ & Lipschtz constant of the transition function in a Lipschitz MDP\\
$L_r$ & Lipschtz constant of the reward function in a Lipschitz MDP\\
$\boldsymbol \varphi_N$ & generic feature map of degree $N$\\
$\varphi_n$ & $n-$th element of a generic feature map $\boldsymbol \varphi_N$\\
$\boldsymbol \varphi_{L,N}$ & Legendre feature map of degree $N$ (1 dimension)\\
$\boldsymbol \varphi_{L,N}^d$ & Legendre feature map of degree $N$ ($d$ dimensions)\\
$N$ & Degree of a feature map\\
$\widetilde N$ & Length of a feature map ($=N$ if $d=1$)\\
$\mathcal I$ & inherent Bellman error\\
$\mathcal \theta$ & Linear parameter in a Linear MDP/ Bellman complete MDP\\
$\text{dim}_E (\mathcal F, \varepsilon)$ & Eluder dimension of the function class $\mathcal F$ with respect to the threshold $\varepsilon>0$\\
$\mathcal N_{\infty}(\mathcal F, \varepsilon)$ & Covering number of the function class $\mathcal F$ with respect to the threshold $\varepsilon>0$ in $L^\infty$ norm
\end{tabular}
\egroup

\clearpage
\section{Omitted Proofs}\label{apx:proofs}

\subsection{Strongly Smooth $\implies$ Weakly Smooth}\label{app:smoothvssmooth}
Let us assume an MDP is Strongly Smooth. Indeed, for every function $f: \Ss\times \As \to \R$ we have:

\begin{align*}
    \mathcal T^*_h f(s,a)&=r_h(s,a)+\E_{s'\sim p_h(\cdot|s,a)}[\max_{a'\in \As}f(s',a')] \\
    & = r_h(s,a)+\int_{\Ss} \max_{a'\in \As}f(s',a')p_h(s'|s,a)\ ds'.\\
\end{align*}
By triangular inequality, this entails:
\begin{align*}
    \|\mathcal T^*_h f\|_{\mathcal C^{\nu,1}} &\le \|r\|_{\mathcal C^{\nu,1}} + \left \|\int_{\Ss} \max_{a'\in \As}f(s',a')p_h(s'|s,a)\ ds'\right \|_{\mathcal C^{\nu,1}},
\end{align*}
where the first term is bounded by assumption, so that we can focus on the second one. 
If we can apply the theorem of exchange between integral and derivative, we have, for every multi-index with $|\bm \alpha|\le \nu$:

\begin{equation}
    D^{\bm \alpha}\int_{\Ss} \max_{a'\in \As}f(s',a') p_h(s'|s,a)\ ds' = \int_{\Ss} \max_{a'\in \As}f(s',a')D^{\bm \alpha} p_h(s'|s,a)\ ds';\label{eq:exchange}
\end{equation}

In such case, using the abbreviation $z=(s,a)$ and $\tilde f(s')=\max_{a'\in \As}f(s',a')$ we get:

\begin{align*}
    D^{\bm \alpha}\int_{\Ss}  \tilde f(s')p_h(s'|z_1)\ ds'-D^{\bm \alpha}\int_{\Ss}  \tilde f(s')p_h(s'|z_2)\ ds' &= \int_{\Ss} \tilde f(s')(D^{\bm \alpha} p_h(s'|z_1)-D^{\bm \alpha} p_h(s'|z_2))\ ds'\\
    & \le \int_{\Ss} \tilde f(s')C_p\|z_1-z_2\|_2\ ds'\\
    &\le \text{Vol}(\Ss)C_p\|z_1-z_2\|_2\|\tilde f\|_\infty\\
    &\le \text{Vol}(\Ss)C_p\|z_1-z_2\|_2\|f\|_\infty\\
    &\le \text{Vol}(\Ss)C_p\|z_1-z_2\|_2\|f\|_{\mathcal C^{\nu,1}}\\
    &= 2^{d_\Ss}C_p\|z_1-z_2\|_2\|f\|_{\mathcal C^{\nu,1}}
\end{align*}


where the second step comes from the Strongly Smoothness assumption on $p_h$ and the third one from the fact that $\max_{s\in \Ss}|\max_{a\in \As}f(s,a)|\le \max_{s\in \Ss}\max_{a\in \As}|f(s,a)|$.
This proves that the norm of the operator $\mathcal T^*_h f$ is bounded by $\text{Vol}(\Ss)C_p\|f\|_{\mathcal C^{\nu,1}}+\|r\|_{\mathcal C^{\nu,1}}$, so that, tanking $C_{\mathcal T^*}=\max \{1,\text{Vol}(\Ss)C_p\}$ we have the boundedness of operator $\mathcal T^*$. It remains to justify Equation \ref{eq:exchange}, by ensuring that we can differentiate under the integral. This holds \cite{folland1999real} under the condition that:

$$\exists g\ge 0,\ \int_\Ss g(s')ds'\le +\infty,\qquad \forall z\quad \left |D^{\bm \alpha} \tilde f(s')p_h(s'|z)\right|\le g(s'),$$

where the derivative is intended w.r.t. $z$, as before. Taking $g(s')=C_p\|f\|_{\mathcal C^{\nu,1}}$ is already sufficient.

\subsection{Lipschitz MDPs are Weakly Smooth but not Strongly Smooth}\label{app:lipMDP}

We only prove that all Lipschitz MDPs are Weakly Smooth, as the fact that they are not always Strongly Smooth can be proved by simply taking a Lipschitz MDPs that is also deterministic; indeed, no smoothness condition can be imposed if $p_h(\cdot|s,a)$ is a Dirac delta.

\begin{thm} 
    Let $M=(\Ss, \As, p, r, H)$ be a Lipschitz MDP with constants $L_r,L_p$. Then, the same MDP is Weakly Smooth for $\nu=0$.
\end{thm}
\begin{proof}
    Let $f\in \mathcal C^{0,1}(\Ss \times \As)$ (which corresponds to the space of Lipschitz functions). We start from:
    \begin{align*}
        \mathcal T^*_h f(s,a)&=r_h(s,a)+\E_{s'\sim p_h(\cdot|s,a)}[\max_{a'\in \As}f(s',a')] \\
        & = r_h(s,a)+\int_{\Ss} \max_{a'\in \As}f(s',a')p_h(s'|s,a)\ ds'.\\
    \end{align*}        
    By triangular inequality we have:
    \begin{align*}
        \|\mathcal T^*_h f\|_{\mathcal C^{0,1}} &\le \|r\|_{\mathcal C^{0,1}} + \left \|\int_{\Ss} \max_{a'\in \As}f(s',a')p_h(s'|s,a)\ ds'\right \|_{\mathcal C^{0,1}}\\
        & = \max\{1,L_r\} + \left \|\int_{\Ss} \max_{a'\in \As}f(s',a')p_h(s'|s,a)\ ds'\right \|_{\mathcal C^{0,1}}.
    \end{align*}
    We have now to evaluate the second part. Indeed, we can bound the infinity norm in this way:
    $$\left \|\int_{\Ss} \max_{a'\in \As}f(s',a')p_h(s'|s,a)\ ds'\right \|_{\infty}\le \|f\|_\infty,$$
    so that we have only to bound its Lipschitz constant to bound its norm in $\mathcal C^{0,1}$. We have:     

    \begin{align*}
        \left |\int_{\Ss} \max_{a'\in \As}f(s',a')p_h(s'|s_1,a_1)\ ds'-\int_{\Ss} \max_{a'\in \As}f(s',a')p_h(s'|s_2,a_2)ds'\right | &= \left |\int_{\Ss} \max_{a'\in \As}f(s',a')\big (p_h(s'|s_1,a_1)-p_h(s'|s_2,a_2)\big )\right | \\
        &\le \wass(p_h(\cdot|s_1,a_1),p_h(\cdot|s_2,a_2))\text{Lip}(\max_{a'\in \As}f)\\
        &\le \wass(p_h(\cdot|s_1,a_1),p_h(\cdot|s_2,a_2))\text{Lip}(f)\\
        &\le \wass(p_h(\cdot|s_1,a_1),p_h(\cdot|s_2,a_2))\|f\|_{\mathcal C^{0,1}}\\
        &\le L_p\|f\|_{\mathcal C^{0,1}}(\|s_1-s_2\|_2+\|a_1-a_2\|_2),
    \end{align*}
    the second passage being valid by definition of Wasserstein distance, the third since the Lipschitz constant of $\max_{a'\in \As}f$ is at most equal to the one of $f$, the fourth by definition of $\|\cdot\|_{\mathcal C^{0,1}}$ norm and the last one by definition of Lipschitz MDP.
    This proves that:

    $$\left \|\int_{\Ss} \max_{a'\in \As}f(s',a')p_h(s'|s,a)\ ds'\right \|_{\mathcal C^{0,1}}\le \max\{1,L_p\}\|f\|_{\mathcal C^{0,1}}.$$

    Overall, this entails:

    $$\|\mathcal T^*_h f\|_{\mathcal C^{0,1}} \le \max\{1,L_r\} + \max\{1,L_p\}\|f\|_{\mathcal C^{0,1}},$$

    which proves the boundedness of the operator $\mathcal T^*_h$.
\end{proof}

\subsection{LinearMDPs are Strongly Smooth}\label{app:linMDP}

\begin{thm}
    Let $M=(\Ss, \As, p, r, H)$ be a Linear MDP with feature map $\bm \varphi$. Then, the same MDP is Strongly Smooth for an order $\nu$ corresponding to the smoothness of $\bm \varphi$.
\end{thm}
\begin{proof}
    By definition, the Linear MDP satisfies, $\forall h\in [H]\ s,s'\in \Ss,\ a\in \As,$ 
    $$r_h(s,a)=\langle \bm \theta_h, \bm \varphi(s,a)\rangle\qquad p_h(s'|s,a)=\langle \bm \mu_h(s'), \bm \varphi(s,a)\rangle.$$
    Assuming that the feature map $\bm \varphi \in \mathcal C^{\nu,1}(\Ss \times \As, \R^{d_{\bm \varphi}})$ (this is the space of functions $\Ss \times \As \to \R^{d_{\bm \varphi}}$ such that each of the $d_{\bm \varphi}$ components is in $\mathcal C^{\nu,1}(\Ss \times \As)$), we have:
    \begin{itemize}
        \item $r_h \in \mathcal C^{\nu,1}(\Ss \times \As)$, being a linear combination of $d_{\bm \varphi}$ functions in $\mathcal C^{\nu,1}(\Ss \times \As)$. Moreover, $\|r_h\|_{\mathcal C^{\nu,1}}\le \|\bm \theta_h\|_1\max_{i=1,\dots,d_{\bm \varphi}}\|\bm \varphi_i\|_{\mathcal C^{\nu,1}}$.

        \item For every $s'$, $p_h(s'|\cdot,\cdot) \in \mathcal C^{\nu,1}(\Ss \times \As)$ for the same reason, and $\sup_{s'\in \Ss}\|p_h(s'|\cdot,\cdot)\|_{\mathcal C^{\nu,1}}\le \sup_{s'\in \Ss}\|\bm \mu(s')\|_1\max_{i=1,\dots,d_{\bm \varphi}}\|\bm \varphi_i\|_{\mathcal C^{\nu,1}}$.
    \end{itemize}
    As, for linear MDPs, it is assumed that $\max \{\|\bm \theta\|_2, \|\bm \mu(s')\|_2\}\le \sqrt {d_{\bm \varphi}}$ \citep{jin2020provably, uehara2021representation}, this ends the proof.
\end{proof}

\subsection{Kernelized MDPs are Strongly Smooth} \label{app:kern}
In this section, we are going to prove that under the cone property, a standard assumption in mathematical analysis, any RKHS with Mat\'ern kernel contains function that are smooth up to a certain degree.
\begin{prop}
    Let $k_m$ be the Matérn kernel of order $m>1$. Then, if the domain $\Omega$ satisfies the cone property (see Definition 1 in \cite{dlotko2014sobolev}), the corresponding RKHS $\mathcal H_{k_m}\subset \mathcal C^{\nu-1,1}(\Omega)$ for every $\nu<m$.
\end{prop}
\begin{proof}
    We will actually prove a stronger statement, that is $\mathcal H_{k_m}\subset \mathcal C^{\nu}(\Omega)$. First, we apply Corollary A.6 from \cite{tuo2016theoretical}, which, under the condition $\lfloor m+d/2\rfloor>d/2$, which is automatically verified being $m>1$, ensures that
    $$\mathcal H_{k_n}\subset W^{m+d/2}(\Omega),$$

    where $W^{m+d/2}(\Omega)$ denotes the Sobolev space of order $m+d/2$, containing functions that have $m+d/2$ derivatives in $L^2(\Omega)$. Therefore, we can apply theorems that embed Sobolev spaces into spaces of continuous functions to get the result. Precisely, from Proposition 2 in \cite{dlotko2014sobolev}, we have, for each $j$ such that $2(m+d/2-j)>d$,

    $$W^{m+d/2}(\Omega) \subset \mathcal C^{j}(\Omega).$$
    
    By taking $j=\nu$, we have $2(m+d/2-\nu)=d+2(m-\nu)>d$, so that

    $$W^{m+d/2}(\Omega) \subset \mathcal C^{\nu}(\Omega),$$

    provided that $\Omega$ satisfies the cone property. This ends the proof. 
\end{proof}

\subsection{Proof of the regret bound for \textsc{Legendre-Eleanor} (Theorem \ref{thm:ele})}\label{app:proofele}

Before coming to the actual proof, it is necessary to introduce a result from approximation theory \citep{schultz1969multivariate,bagby2002multivariate,plesniak2009multivariate}. We start with a simple lemma of functional analysis, which allows us to draw a relation between the space $\mathcal C^{\nu,1}$ we have defined in this article and the more common $\mathcal C^{\nu+1}$.
\begin{lem}\label{lem:convo}
    Let $f\in \mathcal C^{\nu,1}(\R^d)$. Then, for every $\varepsilon > 0$, there is $f_\varepsilon \in \mathcal C^{\nu+1}(\R^d)$ such that
    $$\|f-f_\varepsilon\|_{L^\infty}\le \varepsilon$$
    and $\|f_\varepsilon\|_{\mathcal C^{\nu,1}}\le \|f\|_{\mathcal C^{\nu,1}}$.
\end{lem}
\begin{proof}
    Fix $\varepsilon>0$,and let $\varepsilon'=\|f\|_{\mathcal C^{\nu,1}}^{-1}\varepsilon$. Let $\chi(x)$ be the standard mollifier in $\R^d$:
    $$\chi(x)=
    \begin{cases}
        C\text{exp}\left(\frac{1}{\|x\|^2-1}\right)\qquad &\|x\|<1\\
        0&\|x\|\ge 1
    \end{cases},$$
    For a constant $C$ such that the function integrates to one. 
    If we define $\chi_{\varepsilon'}(x):=\frac{1}{{\varepsilon'}^d}\chi(x/{\varepsilon'})$, we can take
    $$f_{\varepsilon'}(x):=f*\chi_{\varepsilon'}(x)=\int_{\R^d}\ f(y)\chi_{\varepsilon'}(x-y)\ dy.$$

    By the properties of convolution, as $\chi(\cdot)$ is $\mathcal C^\infty$, the function $f_{\varepsilon'}(x)\in \mathcal C^\infty(\R^d)$. Then, we have
    \begin{itemize}
        \item The bound on the norm difference:
        \begin{align*}
            \|f-f_{\varepsilon'}\|_{L^\infty} &=  \|f*\delta_0(x)-f*\chi_{\varepsilon'}(x)\|_{L^\infty}\\
            &\le \|f\|_{\mathcal C^{\nu,1}}\wass(\delta_0(\cdot),\chi_{\varepsilon'}(\cdot))\\
            &= \|f\|_{\mathcal C^{\nu,1}}\int_{\R^d}\|x\|\chi_{\varepsilon'}(x)dx \\
            &\le {\varepsilon'} \|f\|_{\mathcal C^{\nu,1}}=\varepsilon.
        \end{align*}
        Where $\delta_0(\cdot)$ is the Dirac's delta, the second passage is valid by definition of Wassertein distance thanks to the fact that the Lipschitz constant of $f$ is bounded by $\|f\|_{\mathcal C^{\nu,1}}$, and the last is due to the fact that $\chi_{\varepsilon'}(\cdot)$ has integral $1$, support in a ball of radius ${\varepsilon'}$, and center in the origin. 

        \item The bound on the derivatives:
        for every $|\bm \alpha|\le \nu +1$,
        $$\|D^{\bm \alpha} f_{\varepsilon'}\|_{L^\infty} = \|\chi_{\varepsilon'}*D^{\bm \alpha} f\|_{L^\infty}  \le \|D^{\bm \alpha} f\|_{L^\infty},$$
        the last being valid since $\chi_{\varepsilon'}$ has integral one. Therefore,
        $$\|f_{\varepsilon'}\|_{\mathcal C^{\nu,1}} = \max_{|\bm \alpha|\le \nu+1}\|D^{\bm \alpha} f_{\varepsilon'}\|_{L^\infty}\le \max_{|\bm \alpha|\le \nu+1}\|D^{\bm \alpha} f\|_{L^\infty} = \|f\|_{\mathcal C^{\nu,1}}.$$
    \end{itemize}
    This ends the proof. 
\end{proof}

One key ingredient of our next results will be a theorem from approximation theory. This theorem, which is built on a family of result known as Jackson's theorems, ensures that we are able to approximate smooth functions with polynomials, with an error that is lower the more continuous derivative the function has.

\begin{thm}\label{thm:approxtheory}
    For every $\nu,d\in \N$, there is a constant $J_{d,\nu}$ such that for every function $f:[-1,1]^d\to \R$ in $\mathcal C^{\nu,1}([-1,1]^d)$ it holds, for $N>\nu$,
    $$\exists p_N\in \mathcal P_N:\ \|f-p_N\|_{L^\infty} \le 2J_{d,\nu}\|f\|_{\mathcal C^{\nu,1}}N^{-\nu-1},$$
    where $\mathcal P_N$ is the space of multivariate polynomials of degree at most $N$. Moreover, $\|p_N\|_{\mathcal C^{\nu,1}}\le 3J_{d,\nu}\|f\|_{\mathcal C^{\nu,1}}$.
\end{thm}
\begin{proof}    
    This proof will be based on Theorem 1 from \cite{bagby2002multivariate}, which says that, for any function in $\mathcal C^{\nu+1}$ with compact support ($[-1,1]^d$ in our case), there is a polynomial $p_N\in \mathcal P_N$ such that, for every multi-index $\bm \alpha$ such that $|\bm \alpha|\le \nu+1$,
    \begin{equation}\label{eq:bagby}        
    \|D^{\bm \alpha} f-D^{\bm \alpha} p_N\|_{L^\infty}\le J_{d,\nu}N^{|\bm \alpha|-\nu-1}\omega_{f,\nu}(N^{-1}),\end{equation}

    where $J_{d,\nu}$ is a constant, which we can impose to be $>1$ without loss of generality, and $\omega_{f,\nu+1}(\cdot)$ is the $\nu-$modulus of continuity, defined as

    $$\omega_{f,\nu+1}(\delta):=\sup_{|\bm \alpha|<\nu+1}\sup_{\|x-y\|_2\le \delta}|D^{\bm \alpha} f(x)-D^{\bm \alpha} f(y)|.$$

    This theorem cannot be applied on $f$, which is not in $\mathcal C^{\nu+1}([-1,1]^d)$, therefore we apply it on the mollified function $f_\varepsilon$ obtained from Lemma \ref{lem:convo}. Note that

    $$\omega_{f_\varepsilon,\nu+1}(\delta)\le 2\sup_{|\bm \alpha|<\nu+1} \|D^{\bm \alpha} f_\varepsilon\|_{L^\infty}\le 2\|f_\varepsilon\|_{\mathcal C^{\nu,1}} \le 2\|f\|_{\mathcal C^{\nu,1}},$$
    where the last inequality is from Lemma \ref{lem:convo}. This allows us to see that, taking $\bm \alpha = 0$ (the zero multi-index) in Equation \ref{eq:bagby},

    \begin{equation}\|f_\varepsilon- p_N\|_{L^\infty}\le 2\|f\|_{\mathcal C^{\nu,1}}J_{d,\nu}N^{-\nu-1}.\label{eq:zerodiff}\end{equation}
    Moreover, taking the other values of $\bm \alpha$ results in

    \begin{equation}\forall \bm \alpha:\:|\bm \alpha|\le \nu+1,\qquad \|D^{\bm \alpha} f_\varepsilon- D^{\bm \alpha} p_N\|_{L^\infty}\le 2\|f\|_{\mathcal C^{\nu,1}}J_{d,\nu}.\label{eq:onediff}\end{equation}

    These two results allow us to obtain the thesis: applying the triangular inequality to Equation \ref{eq:zerodiff}, by Lemma~\ref{lem:convo},
    $$\|f- p_N\|_{L^\infty} \le \|f_\varepsilon- f\|_{L^\infty}+\|f_\varepsilon- p_N\|_{L^\infty}\le \varepsilon+2\|f\|_{\mathcal C^{\nu,1}}J_{d,\nu}N^{-\nu-1}.$$

    Since this is valid for every $\varepsilon >0$, we obtain $\|f- p_N\|_{L^\infty} \le 2\|f\|_{\mathcal C^{\nu,1}}J_{d,\nu}N^{-\nu-1}$, proving the first statement. As for the second statement, applying the triangular inequality to Equation \ref{eq:onediff} leads to

    \begin{align*}
        \|p_N\|_{\mathcal C^{\nu,1}}& = \max_{|\bm \alpha|\le \nu+1} \|D^\alpha p_N\|_{L^\infty}\\
        & \le \max_{|\alpha|\le \nu+1} \|D^{\bm \alpha} f_\varepsilon\|_{L^\infty}+\|D^{\bm \alpha} f_\varepsilon- D^{\bm \alpha} p_N\|_{L^\infty}\\
        & \le \max_{|\bm \alpha|\le \nu+1} \|D^{\bm \alpha} f_\varepsilon\|_{L^\infty}+2\|f\|_{\mathcal C^{\nu,1}}J_{d,\nu} \qquad\qquad\text{(Equation~\ref{eq:onediff})}\\
        & \le 3\|f\|_{\mathcal C^{\nu,1}}J_{d,\nu},
    \end{align*}    

Where we have used lemma \ref{lem:convo} in the last passage to bound $\max_{|\bm \alpha|\le \nu+1} \|D^{\bm \alpha} f_\varepsilon\|_{L^\infty}$, and also the fact that $J_{d,\nu} \ge 1$.
\end{proof}

The main part of the proof of Theorem~\ref{thm:ele} revolves around showing that Weakly Smooth MDPs paired with Legendre representation map have low inherent Bellman error with respect to some sequence of sets $\mathcal B_h$. Recall the definition of inherent Bellman error:

$$\mathcal I:=\max_{h\in [H]}\sup_{\theta\in \mathcal B_{h+1}}\inf_{\theta'\in \mathcal B_{h}}\|\bm \varphi(s,a)^\top \theta'-\mathcal T^*Q(\theta)(s,a)\|_{L^\infty}.$$

where $Q_\theta(s,a)$ is the function $\bm \varphi(s,a)^\top \theta$. 

\begin{thm}\label{thm:IBE}
    Let $M$ be a Weakly Smooth MDP. Let us consider the pair $(M,\bm \varphi^d_{L,N})$ given by the MDP and the Legendre feature map of degree $N$.
    There is a sequence of compact sets $\mathcal B_{h}\subset \R^{\widetilde N}$ such that the inherent Bellman error of $(M,\bm \varphi^d_{L,N})$ w.r.t. $\{\mathcal B_{h}\}_h$ satisfies, for $N>\nu$,
    $$\mathcal I \le 2J_{d,\nu}C_{\mathcal T^*}\left (\frac{(3J_{d,\nu}C_{\mathcal T^*})^{H}-1}{3J_{d,\nu}C_{\mathcal T^*}-1} +1\right)N^{-\nu-1},$$
    where $J_{d,\nu}$ is the constant from Theorem~\ref{thm:approxtheory}.
\end{thm}
\begin{proof}    
    First thing, we have to define the sequence of compact sets $\mathcal B_{h}\subset \R^{\widetilde N}$. We define

    \begin{equation}\mathcal B_{h}:=\left \{\theta \in \R^{\widetilde N}: \|\bm \varphi_{L,N}^d(\cdot, \cdot)^\top \theta \|_{\mathcal C^{\nu,1}}\le  B(h)\right\},\label{eq:B}\end{equation}
    

    for a constant $B(h)$ to be defined later.
    
    Now, for every $\theta \in \mathcal B_{h+1}$, we have
    
    $$\mathcal T^*Q(\theta)(s,a)=\mathcal T^*\bm \varphi_{L,N}^d(s,a)^\top \theta.$$
    
    Note that, being $\bm \varphi_{L,N}^d(s,a)^\top \theta$ the scalar product between a constant and a vector of polynomials, it is also $\mathcal C^\infty$. Moreover, by definition of $\mathcal B_{h+1}$, we have, for all $\theta\in \mathcal B_{h+1}$,

    $$\|\bm \varphi_{L,N}^d(\cdot, \cdot)^\top \theta \|_{\mathcal C^{\nu,1}}\le B(h+1).$$
    
    So, having assumed that the process is Weakly Smooth of order $\nu$,
    

    \begin{equation}\label{eq:long}
    \|\mathcal T^*\bm \varphi_{L,N}^d(\cdot,\cdot)^\top \theta\|_{\mathcal C^{\nu,1}}\le C_{\mathcal T^*}\left (\|\bm \varphi_{L,N}^d(\cdot, \cdot)^\top \theta\|_{\mathcal C^{\nu,1}} +1\right)\le C_{\mathcal T^*}\left (B(h+1) +1\right).
    \end{equation}
    
    Applying Theorem \ref{thm:approxtheory}, this entails the existence of a polynomial $p_N\in \mathcal P_N$ such that 
    
    \begin{equation}\|\mathcal T^*\bm \varphi_{L,N}^d(\cdot,\cdot)^\top \theta-p_N(\cdot)\|_{L^\infty} \le 2J_{d,\nu}C_{\mathcal T^*}\left (B(h+1) +1\right)N^{-\nu-1}.\label{eq:stratificata}\end{equation}
    
    If we prove that there is $\theta' \in \mathcal B_h$ such that $p_N(\cdot) = \bm \varphi_{L,N}^d(\cdot,\cdot)^\top \theta'$, then we have proved that the inherent Bellman error is bounded by the right hand side of equation \eqref{eq:stratificata}. The fact that this $\theta' \in \mathcal B_h$ exists follows from two considerations:
    \begin{enumerate}
        \item As the set $\{\bm \varphi_{L,N}^d(\cdot,\cdot)_i\}_{i=1}^{\widetilde N}$ is a basis for the vector space of $d-$variate polynomials of degree $N$, there is $\theta'\in \R^{\widetilde N}$ such that $p_N(\cdot) = \bm \varphi_{L,N}^d(\cdot,\cdot)^\top \theta'$.
        \item From Theorem \ref{thm:approxtheory} we also have $\|p_N\|_{\mathcal C^{\nu,1}}\le 3J_{d,\nu}C_{\mathcal T^*}\left (B(h+1) +1\right)$.  
    \end{enumerate}

    Therefore, from the second point, it is sufficient that the value of $B(h)$ in the definition of $\mathcal B_h$ satisfies

    $$B(h)\ge 3J_{d,\nu}C_{\mathcal T^*}\left (B(h+1) +1\right),$$

    which is in particular satisfied by the choice
    \begin{equation}B(h)=\sum_{\tau=1}^{H-h}(3J_{d,\nu}C_{\mathcal T^*})^{\tau} = \frac{(3J_{d,\nu}C_{\mathcal T^*})^{H-h+1}-1}{3J_{d,\nu}C_{\mathcal T^*}-1}-1\label{eq:bh}\end{equation}

    substituting this value into equation \eqref{eq:stratificata}, we get
    $$\mathcal I\le 2J_{d,\nu}C_{\mathcal T^*}\left (\frac{(3J_{d,\nu}C_{\mathcal T^*})^{H}-1}{3J_{d,\nu}C_{\mathcal T^*}-1}\right)N^{-\nu-1},$$
    which ends the proof.
\end{proof}

Now that we have proved Theorem~\ref{thm:IBE}, it is sufficient to apply the results of the literature for the case of MDPs with low inherent Bellman error \citep{zanette2020learning} to achieve a regret bound.
For our convinience, we report here this result
\begin{thm}(Assumption 1 and Theorem 1 from \citep{zanette2020learning})\label{thm:zanette}
    Let $(M,\bm \varphi)$ be a pair MDP-feature map that satisfy the low-inherent Bellmann error assumption with respect to a sequence of sets $\{\mathcal B_h\}_{h=1}^H$ (see \eqref{eq:inherent}). Assume that,
    \begin{enumerate}
        \item $|Q_h^\pi(s,a)|\le 1$ for every $h,s,a$ and every policy $\pi$.
        \item $\|\bm \varphi(s,a)\|_2\le 1$ for every $s,a$.
        \item The reward noise is $1-$subgaussian.
        \item The sets $\{\mathcal B_h\}_{h=1}^H$ are all compact and define $N_h:= \sup_{\theta \in \mathcal B_h}\|\theta\|_{2}^2$.
    \end{enumerate}
    Then, the regret of \textsc{Eleanor} applied on $(M,\bm \varphi)$ satisfies, with probability at least $1-\delta$

    $$R_K\le \bigot \left ( \sum_{h=1}^H N_h\sqrt{K} + \sum_{h=1}^H \sqrt N_h\mathcal I K\right )$$
\end{thm}

We can pass trough this result to achieve a regret bound for every Weakly Smooth MDP.

\begin{thm}\label{thm:preele}
    Let us consider a Weakly Smooth MDP $M$ with state action space $[-1,1]^d$. Under the condition that $\nu>d/2-1$, \textsc{Legendre-Eleanor}, with probability at least $1-\delta$, suffers a regret of order at most:
    $$R_K\le \bigot \left ( C_{\textsc{Ele}}^H \left(\widetilde N \sqrt K + N^{-\nu-1}\sqrt {\widetilde N}K\right )\right),$$
    where the constant depends only on $d$ and $\nu$ and the $\bigot$ hides logarithmic functions of $K$, $\delta$.
\end{thm}

\begin{proof}
    By design of the algorithm, to prove the regret bound we have to show that the couple given by the MDP $(M,\bm \varphi^d_{L,N})$ satisfies the assumptions of theorem \ref{thm:zanette} and then apply its regret bound. Here, we report, point by point, why every assumption is verified.
    \begin{enumerate}
        \item The fact that $|Q_h^\pi(s,a)|\le 1$ is assumed.\\
        \item The feature map satisfies, for every $s\in \Ss,a\in \As,$
        \begin{align*}
            \|\boldsymbol \varphi_{L,N}^d(s,a)\|_2 &= \widetilde N^{-1/2}\left \|\left \{ \varphi_{L,N_1}(x_1)\times \varphi_{L,N_2}(x_2)\dots \varphi_{L,N_d}(x_d): \sum_{i=1}^d N_i\le N\right \}\right\|_2 \\
            & \le \widetilde N^{-1/2}\sqrt{\sum_{i=1}^{\widetilde N}1} = 1.
        \end{align*}
        The last inequality being valid due to the fact that Legendre polynomials are bounded in $[-1,1]$.

        \item Follows from the sub-Gaussianity of the noise and the fact that the state-action value function of every policy is bounded.

        \item This is the only difficult point. We start proving that the sets $\mathcal B_{h}$, defined as in Equation \eqref{eq:B} are compact. Let $B(h)$ be defined as in Equation \eqref{eq:bh}. By the Heine-Borel theorem \cite{rudin1974real}, a subset of $\R^{\widetilde N}$ is compact is and only if it is closed and bounded. We start proving the closure, as we get boundedness from free by the norm inequality proved later. 
        Let $\{\theta_n\}_n \subset \mathcal B_{h}$ such that $\theta_n \to \theta$. Then,

        \begin{align*}
            \|\bm \varphi_{L,N}^d(\cdot,\cdot)^\top (\theta-\theta_n) \|_{\mathcal C^{\nu,1}}& = \max_{|\bm \alpha|\le \nu+1}\| D^{\bm \alpha}\varphi_{L,N}^d(\cdot,\cdot)^\top (\theta-\theta_n) \|_{L^\infty}\\
            & \le \max_{|\bm \alpha|\le \nu+1}\sup_{s,a}|D^{\bm \alpha}\varphi_{L,N}^d(\cdot,\cdot)^\top  (\theta-\theta_n)|\\
            & \le \max_{|\bm \alpha|\le \nu+1}\sup_{s,a}\|D^{\bm \alpha}\varphi_{L,N}^d(\cdot,\cdot)^\top\|_2  \|\theta-\theta_n\|_2\\
            & = \max_{|\bm \alpha|\le \nu+1}\| \|D^{\bm \alpha}\varphi_{L,N}^d(\cdot,\cdot)^\top\|_2  \|_{L^\infty}\|\theta-\theta_n\|_2
        \end{align*}

        where we have used the Cauchy-Schwartz inequality. Now, note that $\varphi_{L,N}^d(\cdot,\cdot)^\top$ is a vector valued function with any component being a Legendre polynomial, so $\mathcal C^\infty$ in particular. Thus, $\max_{|\bm \alpha|\le \nu+1}\| \|D^{\bm \alpha}\varphi_{L,N}^d(\cdot,\cdot)^\top\|_2\|_{L^\infty}$ is bounded and we have
        
        \begin{align*}
            \|\bm \varphi_{L,N}^d(\cdot,\cdot)^\top (\theta-\theta_n) \|_{\mathcal C^{\nu,1}}&\le \underbrace{\max_{|\bm \alpha|\le \nu+1}\| \|D^{\bm \alpha}\varphi_{L,N}^d(\cdot,\cdot)^\top\|_2  \|_{L^\infty}}_{<+\infty}\underbrace{\|\theta-\theta_n\|_2}_{\to 0}\to 0.
        \end{align*}
        
        Moreover, we have, by reverse triangular inequality,

        \begin{align*}
            \|\bm \varphi_{L,N}^d(\cdot,\cdot)^\top \theta\|_{\mathcal C^{\nu,1}} &\le \inf_{n} \|\bm \varphi_{L,N}^d(\cdot,\cdot)^\top (\theta-\theta_n) \|_{\mathcal C^{\nu,1}}+\|\bm \varphi_{L,N}^d(\cdot,\cdot)^\top \theta_n \|_{\mathcal C^{\nu,1}}\\
            & \le \inf_{n} \|\bm \varphi_{L,N}^d(\cdot,\cdot)^\top (\theta-\theta_n) \|_{\mathcal C^{\nu,1}}+B(h)\\
            & = B(h).
        \end{align*}

        Where we have used the fact that $\theta_n \in \mathcal B_{h}$ to bound $\|\bm \varphi_{L,N}^d(\cdot,\cdot)^\top \theta_n \|_{\mathcal C^{\nu,1}}$, and the fact that $\|\bm \varphi_{L,N}^d(\cdot,\cdot)^\top (\theta-\theta_n) \|_{\mathcal C^{\nu,1}} \to 0$ to ensure that $\inf_{n} \|\bm \varphi_{L,N}^d(\cdot,\cdot)^\top (\theta-\theta_n) \|_{\mathcal C^{\nu,1}}=0$. This proves that $\theta \in \mathcal B_{h}$, which means that the set is closed.
        
        The norm inequality follows from the fact that the Legendre polynomials form an orthogonal basis of $L^2(\mathcal \Ss \times \As)$. Indeed we have, by definition, that $\forall \theta \in \mathcal B_{h}, \|\bm \varphi_{L,N}^d(\cdot,\cdot)^\top \theta \|_{L^\infty}\le \|\bm \varphi_{L,N}^d(\cdot,\cdot)^\top \theta \|_{\mathcal C^{\nu,1}}\le B(h)$.
        
        Being in a bounded domain $\Ss \times \As \subset [-1,1]^d$, the $L^\infty$ norm is stronger that the $L^2$ one, and precisely we have $\|\cdot\|_{L^2}\le \sqrt{Vol(\Ss \times \As)}\|\cdot\|_{L^\infty}$. Therefore, we have

        \begin{equation}\forall \theta \in \mathcal B_{h}, \|\bm \varphi_{L,N}^d(\cdot, \cdot)^\top \theta \|_{L^2}\le  \sqrt{Vol(\Ss \times \As)}B(h).\label{eq:Tstar}\end{equation}

        Here the definition Legendre polynomials plays a crucial role: as $\{\bm \varphi_{L,N}^d(\cdot,\cdot)_{i}\}_{i=1}^{\widetilde N}$ is an orthogonal sequence normalized in $L^2$ to $\widetilde N^{-1/2}$, it follows from Parseval's theorem \citep{rudin1974real} on the Hilbert space $L^2(\Ss\times \As)$ that we can bound $\|\bm \varphi_{L,N}^d(\cdot, \cdot)^\top \theta \|_{L^2}$. Indeed,

        \begin{align*}
            \|\bm \varphi_{L,N}^d(\cdot,\cdot)^\top \theta \|_{L^2} & =\left \|\sum_{i=1}^{\widetilde N}[\bm \varphi_{L,N}^d]_i(\cdot,\cdot)\theta_i \right \|_{L^2}\\
            & \overset{(Par)}{=}\sqrt{\sum_{i=1}^{\widetilde N}\|[\bm \varphi_{L,N}^d]_i(\cdot,\cdot)\|_{L^2}^2\theta_i^2}\\
            & =\widetilde N^{-1/2}\sqrt{\sum_{i=1}^{\widetilde N}\theta_i^2}\\
            & =\widetilde N^{-1/2}\|\theta \|_{2}.
        \end{align*}

        Where at passage $(Par)$ we have used Parseval's theorem, exploiting the fact that the $\widetilde N$ components of $\bm \varphi_{L,N}^d$ are all orthogonal in $L^2(\Ss\times \As)$, by definition of Legendre polynomials, and have been normalized to $\widetilde N^{-1/2}$. Note that Parseval theorem works also for infinite components, but here we are considering a function $\varphi_{L,N}^d(\cdot,\cdot)^\top \theta$ which is a linear combination only of the first $\widetilde N$ elements of the Legendre basis, so that all the following components are identically zero.
        
        Substituting this result into Equation \ref{eq:Tstar}, we have

        $$\|\theta \|_{2} \le \sqrt{Vol(\Ss \times \As)}B(h)\widetilde N^{1/2}.$$

        Having the additional term $\sqrt{Vol(\Ss \times \As)}B(h)$ multiplying $\widetilde N^{1/2}$ has the effect of enlarging the regret of the same quantity, which still does not depend on $N$.
    \end{enumerate}

    For this reason, Theorem 1 from \citet{zanette2020learning} results in the following regret bound in high probability
    $$R_K \le \widetilde O\left(\sqrt{Vol(\Ss \times \As)}B(1) \left [H\widetilde N \sqrt K + \mathcal I\sqrt {\widetilde N}K\right ]\right),$$

    which, once Theorem \ref{thm:IBE} is applied to bound the inherent Bellman error, leads to
    
    $$R_K \le \widetilde O\left(\sqrt{Vol(\Ss \times \As)}B(1) \left [H\widetilde N \sqrt K + 2J_{d,\nu}C_{\mathcal T^*}\left (\frac{(3J_{d,\nu}C_{\mathcal T^*})^{H}-1}{3J_{d,\nu}C_{\mathcal T^*}-1} +1\right) N^{-\nu-1}\sqrt {\widetilde N}K\right ]\right).$$

    Grouping all the constants independent of $N$ and $K$, we get the factor
    $$ 2\sqrt{Vol(\Ss \times \As)}B(1) HJ_{d,\nu}C_{\mathcal T^*}\left (\frac{(3J_{d,\nu}C_{\mathcal T^*})^{H}-1}{3J_{d,\nu}C_{\mathcal T^*}-1} +1\right) \le C_{\textsc{Ele}}^H,$$
    for a suitably large constant $C_{\textsc{Ele}}$ only depending on $d$ and $\nu$.
\end{proof}

At this point, it is easy to achieve the regret bound in the form of Theorem \ref{thm:ele}: from a regret bound in the form of theorem \ref{thm:preele}, we can substitute the value of $N=\lceil K^\beta \rceil$ to achieve 

$$R_K\le \bigot \left ( K^{d\beta} \sqrt K + K^{-\beta(\nu+1)}K^{d\beta/2}K\right).$$

Imposing that the exponents are equal leads to $d\beta+\frac{1}{2}=1+\beta(d/2-\nu-1)\implies \beta(d/2+\nu+1)=\frac{1}{2}\implies \beta = \frac{1}{d+2(\nu+1)}$. Substituting in the regret bound, we get precisely

$$R_K\le \bigot \left ( K^{\frac{3d/2+\nu+1}{d+2(\nu+1)}} \right),$$

which is the statement of Theorem \ref{thm:ele}.

\subsection{Lipschitz MDPs have regret bound exponential in $H$}\label{app:expoH}

In this section we prove that every regret bound for algorithms in the Lipschitz MDP setting must grow exponentially with the time horizon $H$. The proof strategy is the following: we start from an instance of a Lipschitz bandit problem with a Lipschitz constant that is exponential in $H$, and show that this can be reduced to a standard Lipschtz MDP (where all Lipschitz constants are independent on $H$). Since it has been shown that the regret bound in a Lipschitz bandit problem is proportional to the Lipschitz constant, this shows that the regret of the Lipschitz MDP is also exponential in $H$. 

\begin{thm}
    The regret in a Lipschtz MDP is at least of order $R_T=\Omega(L_p^{\frac{d(H-2)}{d+2}}K^{\frac{d+1}{d+2}})$.
\end{thm}
\begin{proof}
Let $f:[-1,1]^d\to [-1,1]$ be an $2L_p^{H-2}-$Lipschtz function and $\eta$ a noise bounded in $[-1,1]$.

Define $\widetilde f := (L_p^{-H+2}/2)f, \widetilde \eta := (L_p^{-H+2}/2)\eta$, so that $\widetilde f:[-1,1]^d\to [-L_p^{-H+2}/2,L_p^{-H+2}/2]$ is a $1/2-$Lipschitz function and $\widetilde \eta$ a noise bounded in $[-L_p^{-H+2}/2,L_p^{-H+2}/2]$. Define the following MDP:
\begin{itemize}
    \item The state and action space coincide: $\Ss = \As = [-1,1]^{d/2}$. In this way, $\Ss\times \As = [-1,1]^d$
    \item The starting state is $[0,\dots 0]$ almost surely.

    \item The transition function is defined in the following way:
    \begin{itemize}
        \item For $h=1$, $p_1(s'|s,a)=\delta(s'=a)$, so that the first action becomes the second state.
        \item For $h=2$, $p_2(s'|s,a)=\delta(s'^{(1)}=\widetilde f(s,a)+\widetilde \eta)\prod_{i=2}^{d/2}\delta_0({s'}^{(i)})$, meaning that the next state has the first coordinate equal to $\widetilde f(s,a)$ plus the noise $\eta$, and all the other ones set to zero. Note that this is coherent with the definition of $f$, which goes $[-1,1]^{d/2}=\Ss \times\As \to \R$. 
        \item For $h=2,\dots H$ we have $p_h(s'|s,a)=\delta(s'=L_ps)$, so that the next state is the previous one times a constant (note that, by the bounds on $\widetilde f,\widetilde \eta$, the state never hits the boundary).
    \end{itemize}
    \item The reward function $r_h$ is zero for the first $H-1$ time steps, and $r_H(s,a)=s^{(1)}$, the first component of the state.
\end{itemize}
By definition, it is easy to check that the MDP is Lipschtz with $L_p=L_p$ and $L_r=1$.

In this very peculiar MDP, where only the first two actions $a_1$ and $a_2$ matter, the return can be expresses as a function of them. Precisely, since the reward is only given at the last time step, we have

$$\text{Return}(a_1,a_2)=L_p^{H-2}(\widetilde f(a_1,a_2)+\widetilde \eta)=\frac{1}{2}(f(a_1,a_2)+\eta).$$

In this way, we have shown that the return for this Lipschitz MDP corresponds exact to the feedback in the Lipschitz bandit problem with reward function $\widetilde f/2$ (which is $L_p^{H-2}$-LC) and noise $\eta/2$. 

This shows that any Lipschitz bandit problem with Lipschtz constant $L_P^{H-2}$ can be reduced to to a Lipschtz MDP with constants bounded independently of $H$. Therefore, the regret on the latter problem is as most as high as the one of the former one. As the regret of the latter is well-known to be of order $\Omega(L^{\frac{d}{d+2}}K^{\frac{d+1}{d+2}})=\Omega(L_p^{(H-2)\frac{d}{d+2}}K^{\frac{d+1}{d+2}})$, the proof is complete.
\end{proof}

\subsection{Proof of the regret bound for \textsc{Legendre-LSVI} (Theorem~\ref{thm:lsvi})}\label{app:prooflsvi}

For our convenience, we recall here the main result about Linear MDPs that we are going to use to prove our regret bound.

\begin{thm}\label{thm:chijin}(Assumption B + thm. 3.2 from \cite{jin2020provably}). Let $(M,\bm \varphi)$ a pair MDP-feature map with $\bm \varphi:\Ss\times \As \to \R^{\widetilde N}$, $\bm \mu(\cdot):\Ss\to \R^{\widetilde N}$ a vector of signed measures and $\zeta$ a positive number that satisfy, for any $s,a,s',h$,
\begin{enumerate}
    \item $\|\bm \varphi(s,a)\|_{2}\le 1$
    and $\|\bm \mu(s')\|_{2}\le \sqrt{\widetilde N}$
    \item $\|\bm \theta_h\|_{2}\le \sqrt{\widetilde N}$
    \item $\text{TV}(p_h(\cdot|s,a)-\langle \bm \varphi(s,a), \bm \mu_h(\cdot)\rangle)\le \zeta$
    \item $|r_h(s,a)-\langle \bm \varphi(s,a), \bm \theta_h \rangle|\le \zeta$
\end{enumerate}
Then, algorithm LSVI-UCB satisfies, for every $\delta>0$, with probability at least $1-\delta$ 
    $$R_K\le \bigot(H^{3/2}\widetilde N^{3/2}\sqrt K+\zeta N H K),$$
    where $\bigot$ hides quantities that are logarithmic in $H,K,N,\delta$.
\end{thm}

We now have to prove that any Strongly Smooth MDP equipped with a Legendre feature map becomes a LinearMDP.
\begin{thm}
    Let us consider a Strongly Smooth MDP $M$ with state action space $[-1,1]^d$. Under the condition that $d\le \nu+1$, \textsc{Legendre-LSVI}, with probability at least $1-\delta$, suffers a regret of order at most:
    $$R_K \le \bigot \left(H^{3/2}\widetilde N^{3/2}\sqrt K+H^{3/2}N^{-\nu-1}\widetilde N K \right).$$
    where $\bigot$ hides logarithmic functions of $K$, $\delta$, and $H$.
\end{thm}
\begin{proof}
    By design of the algorithm, to prove the regret bound we have to show that the couple given by the MDP and the Legendre feature map forms a $\zeta-$approximate LinearMDP, so that LSVI-UCB is guaranteed to work.
    To prove that the MDP is a $\zeta-$approximate LinearMDP we have to satisfy the assumptions to apply theorem \ref{thm:chijin}. The first step is to show what are the two components $\bm \mu_h$ and $\theta_h$ in our setting. Indeed, by assuming that the MDP is Strongly Smooth,     
    $$\forall h\in [H]\ \forall s'\in \Ss,\qquad r_h(\cdot,\cdot),p_h(s'|\cdot,\cdot)\in \mathcal C^{\nu,1}(\Ss \times \As),$$
    with $\sup_{h,s'} \|p(s'|\cdot,\cdot)\|_{\mathcal C^{\nu,1}}\coloneqq C_p < \infty$ and $\sup_{h,s'} \|r(\cdot,\cdot)\|_{\mathcal C^{\nu,1}}\coloneqq C_r<+\infty$. Theorem \ref{thm:approxtheory} ensures that

    $$\forall s'\in \Ss,\ \forall h \in [H],\ \exists p_N\in \mathcal P_N:\ \|p_h(s'|\cdot,\cdot)-p_N(\cdot)\|_{L^\infty} \le 2J_{d,\nu}C_pN^{-\nu-1},$$
    $$\forall h \in [H],\ \exists p_N\in \mathcal P_N:\ \|r_h(\cdot,\cdot)-p_N(\cdot)\|_{L^\infty} \le 2J_{d,\nu}C_rN^{-\nu-1}.$$

    As the set $\{\varphi_{L,N}^d(s,a)_i\}_{i=1}^{\widetilde N}$ is a basis for the vector space $\mathcal P_N$ of $d-$variate polynomials of degree $N$, we can define $\bm \mu_h(s')\in \R^{\widetilde N}$ and $\theta_h\in \R^{\widetilde N}$ to be the coefficients such that
    
    \begin{equation}\forall s'\in \Ss,\ \forall h \in [H] :\ \|p_h(s'|\cdot,\cdot)-\bm \varphi_{L,N}^d(\cdot,\cdot)^\top \bm \mu_h(s')\|_{L^\infty} \le 2J_{d,\nu}C_pN^{-\nu-1},\label{eq:lsvi1}\end{equation}
    \begin{equation}\forall h \in [H] :\ \|r_h(\cdot)-\bm \varphi_{L,N}^d(\cdot,\cdot)^\top \theta_h\|_{L^\infty} \le 2J_{d,\nu}C_rN^{-\nu-1}.\label{eq:lsvi2}\end{equation}   
    
    Now we just have to prove that  the four assumtpions of theorem \ref{thm:chijin} hold,
    \begin{enumerate}
        \item The norm bound on the feature map follows exactly as in the proof of Theorem \ref{thm:preele}. For every $s\in \Ss,a\in \As,$
        \begin{align*}
            \|\boldsymbol \varphi_{L,N}^d(s,a)\|_2 &= \widetilde N^{-1/2}\left \|\left \{ \varphi_{L,N_1}(x_1)\times \varphi_{L,N_2}(x_2)\dots \varphi_{L,N_d}(x_d): \sum_{i=1}^d N_i\le N\right \}\right\|_2 \\
            & \le \widetilde N^{-1/2}\sqrt{\sum_{i=1}^{\widetilde N}1} = 1.
        \end{align*}
        The inequality being valid due to the fact that Legendre polynomials are bounded in $[-1,1]$.

        \item Since $\{\bm \varphi_{L,N}^d(s,a)_{i}\}_{i=1}^{\widetilde N}$ is an orthogonal sequence normalized in $L^2$ to $\widetilde N^{-1/2}$, by Parseval's theorem \citep{rudin1974real}, for all $s'\in \Ss, h\in [H]$,

        \begin{align*}
            \|\bm \mu_h(s')\|_2 &= \widetilde N^{1/2}\|\bm \varphi_{L,N}^d(\cdot,\cdot)^\top \bm \mu_h(s')\|_{L^2}\\
            &\le \widetilde N^{1/2}\sqrt{Vol(\Ss\times \As)}\|\bm \varphi_{L,N}^d(\cdot,\cdot)^\top \bm \mu_h(s')\|_{L^\infty}\\
            & \le 2\widetilde N^{1/2}\sqrt{Vol(\Ss\times \As)}\|p_h(s'|\cdot,\cdot)\|_{L^\infty} \le 2\widetilde N^{1/2}C_p\sqrt{Vol(\Ss\times \As)}.
        \end{align*}
        Where equality is by Parseval's theorem (cf. proof of Theorem~\ref{thm:preele}), the first inequality from the $L^2-L^\infty$ norm inequality since  $\Ss\times \As$ is bounded and has finite measure,
        the second inequality from the second part of Theorem \ref{thm:approxtheory}, and the last one by definition of $C_p$. Analogous steps show that
        \begin{align*}
            \|\bm \theta_h\|_2 \le 2\widetilde N^{1/2}C_r\sqrt{Vol(\Ss\times \As)}.
        \end{align*}
        Having proved that $\max_{h\in [H]} \sup_{s'\in \Ss} \{\|\bm \mu_h(s')\|_2,\ \|\bm \theta_h\|_2\}\le 2\max \{C_r, C_p\}\sqrt{Vol(\Ss\times \As)}\widetilde N^{1/2}$, we have that the regret of LSVI-UCB given by theorem \ref{thm:chijin} will just be multiplied by the constant $2\max \{C_r, C_p\}\sqrt{Vol(\Ss\times \As)}$.

        \item To prove the bound on the total variation difference with the transition function, we just need to apply Equation \eqref{eq:lsvi1}:
        \begin{align*}
            \tv\left (p_h(\cdot|s,a),\bm \varphi_{L,N}^d(s,a)^\top \bm \mu_h(\cdot)\right) & \le Vol(\Ss)\sup_{s'\in \Ss} |p_h(s'|s,a)-\bm \varphi_{L,N}^d(s,a)^\top \bm \mu_h(\cdot)|\\
            & \le Vol(\Ss)\sup_{s'\in \Ss} \|p_h(s'|\cdot,\cdot)-\bm \varphi_{L,N}^d(\cdot,\cdot)^\top \bm \mu_h(s')\|_{L^\infty}\\
            &\le 2Vol(\Ss)J_{d,\nu}C_pN^{-\nu-1}.
        \end{align*}

        As usual, the first inequality comes from the $L^1-L^\infty$ norm inequality for finite measure spaces.
        
    \item Lastly, the condition on the difference in norm of reward function and its approximation is proved analogously to the previous point. Indeed, by Equation \ref{eq:lsvi2},
    $$\|r_h-\bm \varphi_{L,N}^d(s,a)^\top \theta_h\|_{L^\infty} \le 2J_{d,\nu}C_rN^{-\nu-1}.$$          
    \end{enumerate}

    All assumptions have been satisfied. Therefore, \textsc{LSVI-UCB} enjoys the regret bound provided by theorem \ref{thm:chijin}, except for the constant factor mentioned in the second point.
\end{proof}

As before, the statement of Theorem \ref{thm:lsvi} follows trivially from the last result. Indeed, we can just replace $N=\lceil K^{\frac{1}{d+2(\nu+1)}} \rceil$ in
$$R_K \le \bigot \left(H^{3/2}\widetilde N^{3/2}\sqrt K+H^{3/2}N^{-\nu-1}\widetilde N K \right),$$

to get
$$R_K \le \bigot \left(H^{3/2}K^{\frac{2d+\nu+1}{d+2(\nu+1)}}\right).$$


\subsection{Proofs from Section \ref{sec:jin}}\label{app:proofjin}

The algorithms described in Section \ref{sec:jin} assume that a function class $\mathcal F$ (composed of possible feature mappings or Hypotheses on $Q^*$) is known. All the regret bounds presented there depend on the covering number of $\mathcal F$ in $L^\infty$ norm (denoted $\mathcal N_{\infty}$), the Eluder dimension \cite{russo2013eluder} of $\mathcal F$ (denoted $\text{dim}_E$), or both. In order to compare our regret bounds with the ones there obtained, we have to assume that $\mathcal{F}=\mathcal F(B)=\{f\in \mathcal C^{\nu,1}(\Ss \times \As): \|f\|_{\mathcal C^{\nu,1}}\le B\}$.
In this way, the algorithms have access to an upper bound $B$ on ``how smooth the MDP is'', which is a small advantage over the standard setting. We present here a two results bounding the Eluder dimension and the covering number of this function class.

The first result, about the Eluder dimension, is of independent interest, as it shows that the regret bounds for Thompson Sampling for continuous spaces, which were proved by \citet{grant2020thompson} only for $d=1$, can be extended to arbitrary dimension. Before coming to the actual theorem about the Eluder dimension, we state a simple result which slightly generalizes Proposition 6 from \cite{russo2013eluder}.

\begin{lem}\label{lem:eluderlinear}
    Define the $(\varepsilon_{\text{left}},\varepsilon_{\text{right}})-$Eluder dimension of a function class $\mathcal F$ as the maximum $n\in \N$ such that there is a sequence $\{x_i\}_{i=1}^n\subset \Xs$ such that
    $$\forall n_0\le n, \exists f_1,f_2\in \mathcal F,\ g=f_1-f_2:\qquad \sum_{i=1}^{n_0-1}g(x_i)^2\le \varepsilon_{\text{left}}^2,\ g(x_{n_0})> \varepsilon_{\text{right}}.$$
    Note that, for $\varepsilon_{\text{left}}=\varepsilon_{\text{right}}$, this dimension corresponds to the standard Eluder. Then, if $\mathcal F$ is a linear class of dimension $N$ (in the linear sense), we have that its $(\varepsilon_{\text{left}},\varepsilon_{\text{right}})-$Eluder dimension is bounded by 
    $$C_0N\frac{2\varepsilon_{\text{left}}^2+\varepsilon_{\text{right}}^2}{{\varepsilon_{\text{right}}}^2}\left [\log\left (\frac{2\varepsilon_{\text{left}}^2+\varepsilon_{\text{right}}^2}{{\varepsilon_{\text{right}}}^2}\right)+\log(1+\varepsilon_{\text{left}}^{-2})+C_1\right]$$
    for some constant $C$.
\end{lem}
\begin{proof}
    We follow the proof of Proposition 6 from \cite{russo2013eluder}. The only change occurs in the last step where, in the fraction $\frac{1+x}{x}$ we set $x=\frac{\varepsilon_{\text{right}}^2}{2\varepsilon_{\text{left}}^2}$ instead of $x=\frac{1}{2}$.
\end{proof}

\begin{thm}\label{thm:eluder}
    Let $\mathcal F(B)=\{f\in \mathcal C^{\nu,1}(\Ss \times \As): \|f\|_{\mathcal C^{\nu,1}}\le B\}$. Then, for every $\varepsilon>0$, $\text{dim}_E (\mathcal F, \varepsilon) = \bigot(B^{\frac{d}{\nu+1}}\varepsilon^{-d/(\nu+1)})$.
\end{thm}
\begin{proof}
    Let $\mathcal G=\mathcal F(B)-\mathcal F(B)\subset \mathcal F(2B)$. By definition, to prove that $\mathcal F$ has an $\varepsilon-$Eluder dimension bounded by $n$ corresponds to proving that there are no points $x_1,\dots x_n \in [-1,1]^d$ such that
    $$\forall n_1\le n\ \exists g\in \mathcal G\qquad \sum_{i=1}^{n_1-1}g(x_i)^2\le \varepsilon^2,\qquad g(x_{D_1})\ge \varepsilon.$$
    We can reason as follows. Let us divide $[-1,1]^d$ into disjoint hypercubes of side $1/\ell$, for $\ell\in \N$. This can be done with exactly $\ell^d$ hypercubes, that we are going to call $\{C_j\}_{j=1}^{\ell^d}$. We start by proving that, under each of the hypercubes, each function of $\mathcal G$ is almost linear in some features.
    
    To prove it, note that letting $g \in \mathcal G$, Definition 1 in \cite{liu2021smooth}~\citep[originally from][]{tsybakov2008introduction} ensures that for any fixed $y\in [-1,1]^d$,
    $$\forall x\in [-1,1]^d,\qquad |g(x)-T_y(x)|\le 2B\|x-y\|_{\infty}^{\nu+1},\qquad T_y[g](x)=\sum_{|\bm \alpha|\le \nu} \frac{D^{\bm \alpha}g(y)}{\bm \alpha!}(x-y)^{\bm \alpha},$$
    where the term $T_y(\cdot)$ corresponds to the Taylor polynomial of order $\nu$ centered in $y$. If we apply this to $y=y_{j^\star}$, the middle points of $C_{j^\star}$, we have that, for every $x\in C_{j^\star}$ 
    \begin{align*}
        |g(x)-T_{y_{j^\star}}[g](x)| &\le 2B\|x-y_{j^\star}\|_{\infty}^{\nu+1}\\
        & \le 2B(2\ell)^{-\nu-1},
    \end{align*}
    where the last inequality comes from the fact that any point of $C_{j^\star}$ cannot have an $\ell_\infty$ distance more than $1/(2\ell)$ form the center of the hypercube. Being valid for every $x\in C_{j^\star}$, this result entails that
    \begin{equation}\|g(\cdot)-T_{y_{j^\star}}[g](\cdot)\|_{L^\infty(C_{j^\star})}\le 2B(2\ell)^{-\nu-1}.\label{eq:taylor}\end{equation}
    At this point, assuming that there are $n_{\text{ind}}$ $\varepsilon-$independent points in $C_{j^\star}$ corresponds to assume that there is a sub-sequence $\{x_{i_k}\}_{k=1}^{n_{\text{ind}}}\subset \{x_i\}_{i=1}^n\cap C_{j^\star}$ such that
    \begin{equation}\forall k_0\le n_{\text{ind}}\qquad \exists g\in \mathcal G: \sum_{k=1}^{k_0-1}g(x_{i_k})^2\le \varepsilon^2,\qquad g(x_{k_0})\ge \varepsilon.\label{eq:geluder}\end{equation}
    Still, applying Equation \eqref{eq:taylor}, we have that, whichever the choice of $\{x_{i_k}\}_{k=1}^{n_{\text{ind}}}$, denoting with $\|g\|_{p,x}$ the norm $p-$of $g$ under the set $\{x_{i_k}\}_{k=1}^{n_{\text{ind}}}$,
    \begin{align*}\sqrt{\sum_{k=1}^{k_0-1}g(x_{i_k})^2}-\sqrt{\sum_{k=1}^{k_0-1}T_{y_{j^\star}}[g](x_{i_k})^2}&=\|g(\cdot)\|_{2,x}-\|T_{y_{j^\star}}[g](\cdot)\|_{2,x}\\
    &\le\|g(\cdot)-T_{y_{j^\star}}[g](\cdot)\|_{2,x}\\
    &\le \sqrt{n_{\text{ind}}}\|g(\cdot)-T_{y_{j^\star}}[g](\cdot)\|_{\infty,x}\\
    & \le 2B\sqrt{n_{\text{ind}}}(2\ell)^{-\nu-1},
    \end{align*}
    where the first step is by the definition of $2-$norm, the second from the triangular inequality, the third from bounding the $2$-norm with the $\infty-$norm, and the last one from Equation~\eqref{eq:taylor}. From this follows that any choice $\{x_{i_k}\}_{k=1}^{n_{\text{ind}}}$ satisfying Equation \eqref{eq:geluder} must also satisfy
    \begin{equation}\forall k_0\le n_{\text{ind}}\qquad \exists g\in \mathcal G: \sum_{k=1}^{k_0-1}T_{y_{j^\star}}[g](x_{i_k})^2\le (\underbrace{\varepsilon+2B\sqrt{n_{\text{ind}}}(2\ell)^{-\nu-1}}_{\varepsilon_{\text{left}}})^2,\qquad T_{y_{j^\star}}[g](x_{k_0})\ge \underbrace{\varepsilon-2B(2\ell)^{-\nu-1}}_{\varepsilon_{\text{right}}}.\label{eq:tayloreluder}\end{equation}

    Equation \eqref{eq:tayloreluder} corresponds to the $(\varepsilon_{\text{left}},\varepsilon_{\text{right}})-$Eluder dimension (defined in lemma \ref{lem:eluderlinear}) of the function class $T_{y_{j^\star}}[g](\cdot):\ g\in \mathcal G$, with $\varepsilon_{\text{left}}=\varepsilon+2B\sqrt{n_{\text{ind}}}(2\ell)^{-\nu-1}$ and $\varepsilon_{\text{right}}=\varepsilon-2B(2\ell)^{-\nu-1}$. For the rest, note that, by definition, $\{T_{y_{j^\star}}[g](\cdot):\ g\in \mathcal G\}$ is a subset of a vector space spanned by $\{(\cdot-y)^{\bm \alpha}\}_{\bm \alpha}$ for every multi-index $|\bm \alpha|\le \nu$. The dimension (in the linear sense) of this vector space corresponds to $\binom{\nu+d}{d}$.

    Therefore, Lemma \ref{lem:eluderlinear} (a slight modification of Proposition 6 from \cite{russo2013eluder}) ensures that Equation \eqref{eq:tayloreluder} is only possible when

    \begin{align*}n_{\text{ind}}&\le C_0N\frac{2\varepsilon_{\text{left}}^2+\varepsilon_{\text{right}}^2}{{\varepsilon_{\text{right}}}^2}\left [\log\left (\frac{2\varepsilon_{\text{left}}^2+\varepsilon_{\text{right}}^2}{{\varepsilon_{\text{right}}}^2}\right)+\log(1+\varepsilon_{\text{left}}^{-2})+C_1\right]\\
    &\le C_0Nf_{\text{log}}\left (\frac{2\varepsilon_{\text{left}}^2+\varepsilon_{\text{right}}^2}{{\varepsilon_{\text{right}}}^2}+\log(1+\varepsilon_{\text{left}}^{-2})+C_1\right)\end{align*}

    Where, for readability, we have called $f_{log}(\cdot)=\cdot \log(\cdot)$ and $C_2=C\binom{\nu+d}{d}$. Letting $\rho:=2B(2\ell)^{-\nu-1}$, the last quantity is bounded by

    $$n_{\text{ind}}\le C_2Nf_{\text{log}}\left (\frac{3(\varepsilon+\sqrt{n_{\text{ind}}}\rho)^2}{(\varepsilon-\rho)^2}+\log(1+\varepsilon^{-2})+C_1\right).$$

    If we take $\rho = \varepsilon/\max\{2,K\}$, for a constant $K$ to be decided later, we get that

    \begin{align*}
        n_{\text{ind}}&\le C_2f_{\log}\left(\frac{3(\varepsilon+\sqrt{n_{\text{ind}}}\rho)^2}{(\varepsilon-\rho)^2}+\log(1+\varepsilon^{-2})+C_1\right)\\
        &\le f_{\log}\left(C_2\frac{3\varepsilon^2(1+\sqrt{n_{\text{ind}}}/K)^2}{(\varepsilon/2)^2}+\log(1+\varepsilon^{-2})+C_1\right)\\
        &\le f_{\log}\left(12C_2(1+\sqrt{n_{\text{ind}}}/K)^2+\log(1+\varepsilon^{-2})+C_1\right).
    \end{align*}

    If we take $K=\lceil\sqrt{1+f_{\log}\left(48C_2+\log(1+\varepsilon^{-2})+C_1\right)}\rceil$, we can see that the previous inequality does not hold for $n_{\text{ind}}=K^2$. Indeed,

    \begin{align*}
        n_{\text{ind}}&\le f_{\log}\left(12C_2(1+\sqrt{n_{\text{ind}}}/K)^2+\log(1+\varepsilon^{-2})+C_1\right)\\
        &\le f_{\log}\left(12C_2(1+1)^2+\log(1+\varepsilon^{-2})+C_1\right)\\
        &=f_{\log}\left(48C_2\log(1+\varepsilon^{-2})+C_1\right)\\
        &<K^2=n_{\text{ind}}.
    \end{align*}

    Therefore, this tells us that, for $\rho = \varepsilon/\max\{2,K\}$ and $K=\lceil\sqrt{1+f_{\log}\left(48C_2+\log(1+\varepsilon^{-2})+C_1\right)}\rceil$ the number of independent points cannot be (exactly) $K^2$. This entails, by definition of independence, that $n_{\text{ind}}$ also cannot be any number higher than $K^2$, as a longer sequence would entail that the first $K^2$ points are also independent.

    With the previous reasoning, we have shown that, for $\rho = \varepsilon/\max\{2,K\}$, each set $C_{j}$ cannot contain more than $K^2$ $\varepsilon-$independent points. Therefore, the same holds for any subset of $C_{j}$, corresponding to $\rho\le\varepsilon/\max\{2,K\}$, which translates in the following bound on $\ell$
    $$\varepsilon/\max\{2,K\} \ge  2B(2\ell)^{-\nu-1}\implies \ell\ge \frac{1}{2}\left(\frac{2B\max\{2,K\}}{\varepsilon}\right)^{\frac{1}{\nu+1}}.$$

    We can just take 
    $$\ell = \left \lceil \frac{1}{2}\left(\frac{2B\max\{2,K\}}{\varepsilon}\right)^{\frac{1}{\nu+1}}\right \rceil,$$

    which implies a total number of hypercubes given by $\ell^d$:

    $$\ell^d = \left \lceil \frac{1}{2}\left(\frac{2B\max\{2,K\}}{\varepsilon}\right)^{\frac{1}{\nu+1}}\right \rceil^d\le \left(\frac{2B\max\{2,K\}}{\varepsilon}\right)^{\frac{d}{\nu+1}}.$$

    In the end, we have proved that, dividing the space $[-1,1]^d$ into this number of hypercubes, no hypercube can contain more than $n_{\text{ind}}$ points that form a $\varepsilon-$independent sequence. Thus, the full length $n$ of $\{x_i\}_{i=1}^n$ is bounded by
    $$n\le n_{\text{ind}}\left(\frac{2B\max\{2,K\}}{\varepsilon}\right)^{\frac{d}{\nu+1}}\le K^2\left(\frac{2B\max\{2,K\}}{\varepsilon}\right)^{\frac{d}{\nu+1}},$$

    with $C_2=C\binom{\nu+d}{d}$ and $K=\lceil\sqrt{1+f_{\log}\left(48C_2+\log(1+\varepsilon^{-2})+C_1\right)}\rceil$. As all terms $B,C,\nu$ are constants not depending on $\varepsilon$, while $K$ depends on it logathimically, the proof is complete.

\end{proof}

As side effect of this theorem, we generalize to the multidimensional case Theorem 3 from
\cite{grant2020thompson}. Thus, Thompson Sampling (as described in their section 1.2) can now be shown to have a regret guarantee in dimension higher than one.

\begin{thm}\label{thm:cover}
    Let $\mathcal F(B)=\{f\in \mathcal C^{\nu,1}(\Ss \times \As): \|f\|_{\mathcal C^{\nu,1}}\le B\}$. Then, for every $\varepsilon>0$,
    $\log(\mathcal N_{\infty}(\mathcal F, \varepsilon)) = \bigo(B^{\frac{d}{\nu+1}}\varepsilon^{-d/(\nu+1)})$.
\end{thm}
\begin{proof}
    Example 4 of \cite{russo2013eluder} shows that, if $\mathcal G$ is a vector space of dimension $\widetilde N$,
    $$\log(\mathcal N_{\infty}(\mathcal G, \varepsilon))= \bigo(\widetilde N \log(\varepsilon^{-1})).$$
    Unfortunately, the dimension of $\mathcal F(B)$, viewed as a subset of a vector space, is $+\infty$. Nonetheless, we can rely on our Theorem \ref{thm:approxtheory} to achieve a non-vacuous bound. The latter ensures that, for any $f\in \mathcal F(B)$,
    $$\exists p_N\in \mathcal P_N:\ \|f-p_N\|_{L^\infty} \le 2J_{d,\nu}BN^{-\nu-1},$$
    where $\mathcal P_N$ is the space of multivariate polynomials of degree at most $N$. To ensure $\varepsilon/2 = \|f-p_N\|_{L^\infty}$, we need
    $$N\ge \left (\frac{4J_{d,\nu}B}{\varepsilon}\right)^{1/(\nu+1)}.$$
    Under this condition, we every $\varepsilon/2-$cover for $\mathcal P_N$ corresponds to a $\varepsilon-$cover for $\mathcal F$, so that
    $$\text{dim}_E (\mathcal F, \varepsilon)\le \text{dim}_E (\mathcal P_N,\varepsilon/2).$$
    Therefore, as the space $\mathcal P_N$ is in fact a vector space of dimension $\tilde N=\binom{N+d}{N}\approx N^d$, we have 
    $$\log(\mathcal N_{\infty}(\mathcal F, \varepsilon))\le \log(\mathcal N_{\infty}(\mathcal P_N, \varepsilon/2)) = \bigo(N^d\log(\varepsilon^{-1})) = \bigot(B^{\frac{d}{\nu+1}}\varepsilon^{-d/(\nu+1)}).$$
\end{proof}

With this theorem, we can show that some very recent algorithms achieve a nontrivial regret bound for our setting.

\begin{prop}[Restatement of Theorem~\ref{thm:freelunch}]
    Let us assume we run Algorithm 1 from \cite{ren2022free} on a Strongly Smooth MDP such that the transition dynamics is Gaussian, in the sense that at any time step $h$ we have $s_{h+1}=f(s_h,a_h)+\epsilon_h$, where $\epsilon_h \sim \mathcal N(0,\Sigma)$. Then, provided that the algorithm knows an upper bound $B$ on $\|f\|_{\mathcal C^{\nu,1}}$, and $d<\nu+1$, its regret is bounded, with probability at least $1-\delta$, by
    $$R_K\le \bigot \left(B^{\frac{d}{\nu+1}}H^{\frac{3\nu+d+3}{2\nu+2}}K^{\frac{\nu+d+1}{2\nu+2}}\right),$$
    where the $\bigot$ hides a quantity that is logarithmic in $K,H,B,\delta$.
\end{prop}
\begin{proof}
    Under our assumptions we can apply Theorem 5 from \cite{ren2022free} for $\mathcal F=\{f\in \mathcal C^{\nu,1}(\Ss \times \As): \|f\|_{\mathcal C^{\nu,1}}\le B\}$. This gives
    $$R_K\le \bigot \left(H^{3/2}\sqrt{K\text{dim}_E (\mathcal F, (HK)^{-1/2})\log(\mathcal N_{2}(\mathcal F, (HK)^{-1/2}))}\right).$$
    Then, by our Theorems \ref{thm:eluder} and \ref{thm:cover},
    $$R_K\le \bigot \left(B^{\frac{d}{\nu+1}}H^{3/2}\sqrt{K^{1+\frac{d}{\nu+1}}H^{\frac{d}{\nu+1}}}\right)=\bigot \left(B^{\frac{d}{\nu+1}}H^{\frac{3\nu+d+3}{2\nu+2}}K^{\frac{\nu+d+1}{2\nu+2}}\right).$$
\end{proof}

\begin{prop}[Restatement of Theorem~\ref{thm:golf}]
    Let us assume we run algorithm \textsc{Golf} on a Strongly Smooth MDP. Then, provided that the algorithm knows the constant $C_{\mathcal T^*}$ and $d\le \frac{2}{3}\nu+\frac{2}{3}$, its regret is bounded, with probability at least $1-\delta$, by
    $$R_K\le \bigot \left(C_{\mathcal T^*}^{\frac{dH}{\nu+1}}K^{\frac{2\nu+3d+2}{4\nu+4}}\right),$$

    where the $\bigot$ hides a quantity that is logarithmic in $K,C_{\mathcal T^*},\delta$.
\end{prop}
\begin{proof}
    We briefly recall the assumptions of the regret bound for \textsc{Golf} given by \cite{jin2021bellman}
    \begin{enumerate}
        \item Assumption 1: the algorithm knows a set $\mathcal F=\{\mathcal F_1,\dots \mathcal F_H\}$ containing one function class for each time step such that $\forall h, Q_h^*\in \mathcal F_h$.
        \item Assumption 2: The class $\mathcal F$ is closed under Bellmann optimality operator: $\forall f\in \mathcal F_{h+1},\ \mathcal T^*f\in \mathcal F_{h}$.
    \end{enumerate}
    Knowing $C_{\mathcal T^*}$, it is possible to define a function class $\mathcal F$ containing all possible candidates for the $Q^*$ function of the MDP. In fact, by the definition of Bellman optimality operator, we can see that, for every $h$, 
    
    $$\|Q_h^*\|_{\mathcal C^{\nu,1}}= \|\mathcal T^*Q_{h+1}^*\|_{\mathcal C^{\nu,1}} \le  C_{\mathcal T^*} (1+\|Q_{h+1}^*\|_{\mathcal C^{\nu,1}})\le \frac{C_{\mathcal T^*}^{H-h+1}-1}{C_{\mathcal T^*}-1}.$$
    
    Therefore, using the function classes $\mathcal F_h=\{f\in \mathcal C^{\nu,1}(\Ss \times \As): \|f\|_{\mathcal C^{\nu,1}}\le \frac{C_{\mathcal T^*}^{H-h+1}-1}{C_{\mathcal T^*}-1}\}$ allows to satisfy both assumptions 1 and 2 from \cite{jin2021bellman}.
    
    Theorem 15 from the same paper ensures that

    $$R_K\le \bigot \left(H\sqrt{K\text{dim}_{BE} (\mathcal F, K^{-1/2})\log(\mathcal N_{\infty}(\mathcal F, K^{-1}))}\right),$$

    a result which is again bounded by Proposition 12 (again from the same paper)

    $$R_K\le \bigot \left(H\sqrt{K\max_{h=1,\dots H}\text{dim}_{E} (\mathcal F_h, K^{-1/2})\log(\mathcal N_{\infty}(\mathcal F, K^{-1}))}\right).$$

    Then, by our Theorems \ref{thm:eluder} and \ref{thm:cover}, we can bound $\max_{h=1,\dots H}\text{dim}_{E} (\mathcal F_h, K^{-1/2})\le \text{dim}_{E} (\mathcal F_H, K^{-1/2})\le \bigot(B^{\frac{d}{\nu+1}}K^{d/(2\nu+2)})$ and $\log(\mathcal N_{\infty}(\mathcal F, K^{-1}))\le \bigot(HB^{\frac{d}{\nu+1}}K^{d/(\nu+1)})$ where $B=\frac{C_{\mathcal T^*}^{H-h+1}-1}{C_{\mathcal T^*}-1}$. Substituting this result, we get
    $$R_K\le \bigot \left(C_{\mathcal T^*}^{\frac{dH}{\nu+1}}H\sqrt{K^{1+\frac{3d/2}{\nu+1}}}\right)=\bigot \left(C_{\mathcal T^*}^{\frac{dH}{\nu+1}}K^{\frac{2\nu+3d+2}{4\nu+4}}\right),$$

    which ends the proof.
\end{proof}

\section{Details of the Numerical Simulation}\label{app:num}

In section \ref{sec:expe}, we have performed a numerical simulation on a modified version of the Linear Quadratic Regulator (LQR). Both environments took the form
$$s_{h+1} = g(As_h + Ba_h + \xi_h),$$
$$r_{h} = -s_h^\top Qs_h - a_h^\top Ra_h,$$
where $g(x):=\frac{x}{1+\|x\|_2}$. Moreover, in both cases the dimension of the state space corresponds to $2$, and the one of the action space to $1$. Also, we have in both cases
$$B=\begin{bmatrix}
    1\\
    1
\end{bmatrix}
\qquad Q=\begin{bmatrix}
    1 & 0\\
    0 & 1
\end{bmatrix}\qquad R=[0.2].$$

what changes is the matrix $A$, which determines most of the dynamics of the system. For this matrix, we have
$$\text{Left experiment:}\ A=
\begin{bmatrix}
    0.7 & 0.7\\
    -0.7 & 0.7
\end{bmatrix}\qquad 
\text{Right experiment:}\ A=
\begin{bmatrix}
    0 & 1\\
    1 & 0
\end{bmatrix}.$$

\subsection{Practical details}

Finally, we report some the details on how the computation was performed in the paper. These are important to ensure the truthfulness of the results and the claims based on empirical validation.

\paragraph{Training Details.} The algorithms were implemented in \textsc{Python3.7}. Each experiment was executed using five random seeds (corresponding to the first five natural numbers), and the computations were distributed across five parallel processes using the \textsc{joblib} library. The total computational time for the first experiment was of $189935$ seconds, more than two days and four hours.

\paragraph{Compute.} We used a server with the following specifications:
\begin{itemize}
    \item CPU: \textsc{88 Intel(R) Xeon(R) CPU E7-8880 v4 @ 2.20GHz cpus},
    \item RAM: \textsc{94,0 GB}.
\end{itemize}

As mentioned, we parallelized the computing for the five different random seeds, therefore only five of the $88$ cores were actually used.

\end{document}